\documentclass[twoside]{article}

\usepackage[preprint]{aistats2027}
\usepackage[authoryear,round]{natbib}

\usepackage[utf8]{inputenc} 
\usepackage[T1]{fontenc}    
\usepackage{url}            
\usepackage{booktabs}       
\usepackage{amsfonts}       
\usepackage{nicefrac}       
\usepackage{microtype}      
\usepackage{amsmath}
\usepackage{amssymb}
\usepackage{amsthm}
\usepackage{graphicx}
\usepackage{mathtools}
\usepackage{algorithm}
\usepackage{algpseudocode}
\usepackage{wrapfig}
\usepackage{changepage}

\usepackage[colorlinks=true]{hyperref}     
\usepackage{xcolor}
\definecolor{linkblue}{RGB}{0,0,128}
\hypersetup{linkcolor=linkblue,citecolor=linkblue,urlcolor=linkblue}
\definecolor{terracotta}{HTML}{9D5242}

\newcommand{\tc}[1]{\textcolor{blue}{[TC: #1]}}

\usepackage{mathtools}
\mathtoolsset{showonlyrefs=true}

\theoremstyle{plain}
\newtheorem{theorem}{Theorem}[section]
\newtheorem{proposition}[theorem]{Proposition}

\theoremstyle{definition}

\theoremstyle{remark}

\begin{document}

%

%

\twocolumn[

\aistatstitle{Feature Space Adaptation for Effortless Gaussian Process Flows}


\begin{adjustwidth}{-0.2in}{-0.2in}
\aistatsauthor{
  Thomas Cowperthwaite
  \And Louis Sharrock
  \And Lachlan Astfalck
  \And Henry Moss
}
\aistatsaddress{
  University of Cambridge\\tc656@cam.ac.uk
  \And University College London
  \And University of New South Wales
  \And Lancaster University
}
\end{adjustwidth}


]

\begin{abstract}
    Outside the linear-Gaussian regime, conditional sampling from Gaussian processes (GPs) is challenging. Recent methods such as FlowGP \citep{moss_conditioning_2026} can condition on arbitrary non-linear and non-Gaussian statements, but at considerable cost: an expensive iterative and high-dimensional diffusion that requires hand-specified kernel hyperparameters. In this paper, we alleviate two significant drawbacks of FlowGP by (1) introducing kernel approximations that enable scaling to high-resolution domains and (2) proposing a way to obtain the marginal likelihood by measuring the work needed to steer the diffusion towards conditioning statements. We enable, for the first time, hyperparameter optimisation within FlowGP and demonstrate our approach on probabilistic downscaling from areal summary statistics, PDE solution inference on irregular domains, and recovery of sea level anomaly fields from non-Gaussian satellite observations.


\end{abstract}

\section{INTRODUCTION}
Gaussian processes (GPs) provide flexible probabilistic models over unknown functions, with calibrated uncertainty and a natural means of incorporating prior knowledge \citep{rasmussen_gaussian_2006}.
Conditioning on data is analytically tractable for linear observations with Gaussian noise.
However, scientific information often takes more complex forms, including non-linear differential equations, integral measurements, inequalities, and boundary constraints. 
Existing solutions typically rely on posterior approximations \citep{rue2007approximate, cseke2011approximate, hensman2013gaussian}, or specialised kernels and inference schemes derived for particular classes of constraints \citep{hamelijnck2024physics, long_autoip_2022, steinert2025integration}.
Recently, \cite{moss_conditioning_2026} introduced \textsc{FlowGP}, a general framework for conditioning GPs on \textit{non-linear} and \textit{non-Gaussian} statements, without the need for bespoke inference schemes.
%

Despite their promise, however, GP flows have several important limitations.
First, the flow operates over a fixed discretisation of the function's domain, and thus the computational requirements grow with the number of evaluation points $N$, including an initial $\mathcal{O}(N^3)$ matrix factorisation and an $N$-dimensional ODE solve.
This becomes restrictive when the number of evaluation points is large, such as for fine spatial or spatio-temporal grids.
%
Large computational cost further makes repeated computation across, and thus optimisation of, kernel hyperparameters prohibitively expensive.
%
Finally, known structure such as boundary conditions must be represented via conditioning statements, which makes the sampling diffusion challenging to solve \citep{moss_conditioning_2026}.

\begin{figure*}[t]
    \centering
    \includegraphics[width=\linewidth]{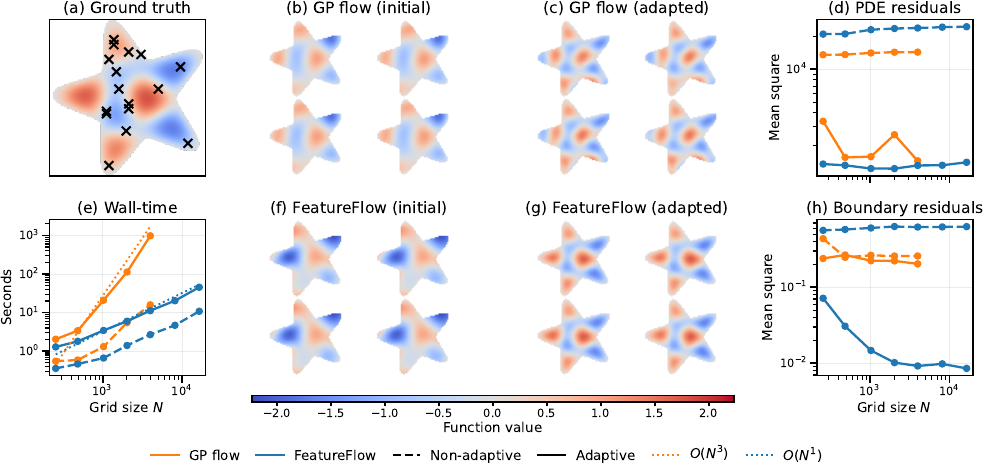}
    \caption{Gaussian process samples, conditioned on non-conjugate information, i.e., satisfying a non-linear PDE. \textbf{(a)} The PDE domain, equipped with a homogeneous Dirichlet boundary condition, and noisy Gaussian observations. Samples drawn from a GP flow \textbf{(b, c)} and \textsc{FeatureFlow}, first using deliberately misspecified kernel hyperparameters \textbf{(b, f)}, and secondly using the adaptive scheme \textbf{(c, g)}. \textsc{FeatureFlow} obtains higher-quality samples, as measured by condition residual \textbf{(d, h)}, and with more than an order of magnitude speed-up \textbf{(e)}.}
    \label{fig: pde-results}
\end{figure*}

In this work, we introduce \textsc{FeatureFlow}, which extends GP flows to operate on the weights of a finite-feature GP approximation instead of function values.
The resulting flow has dimension $R \ll N$, where $R$ denotes the (tunable) number of features, and decouples the cost of the flow from the resolution at which its function samples are evaluated.
This makes kernel hyperparameter optimisation practical and the feature representation also provides access to analytic derivatives which may be further exploited.
We evaluate the resulting method through wall-time scaling and numerical accuracy, across two important problem classes: probabilistic downscaling \citep{perezhogin_generative_2023, schmidt_generative_2025} and inference of PDE solutions on irregular domains, as shown in Fig. \ref{fig: pde-results} \citep{khara_neural_2022, zhao_diffeomorphism_2025}.
We also demonstrate the real-world utility of scalable flow-based inference by applying \textsc{FeatureFlow} to a well-studied, challenging non-linear inverse problem in oceanography ~\citep{fu_satellite_2001, passaro_potential_2023}.

\subsection{Related Work}
\label{sec: rel-work}

\textbf{Scalable GP representations. }The computational cost of exact Gaussian process inference has motivated a large literature on scalable approximations, including inducing-point methods \citep{titsias_variational_2009}, structured kernel approximations \citep{saatci_scalable_2011}, and finite-dimensional feature representations.
Examples of finite-feature respresentations include Fourier features adapted to bounded domains \citep{hensman_variational_2018}, Laplacian eigenfunctions that incorporate domain geometry and boundary conditions \citep{solin_know_2019}, and spherical harmonics for functions defined on spherical domains \citep{dutordoir_sparse_2020}.
Beyond scalability, each basis imposes constraints on encoded functions---many of which are common in scientific problems.

\textbf{Non-conjugate GP inference. } Outside of linear-Gaussian conditioning, exact GP posterior inference is generally intractable and is commonly addressed using general approximations such as the Laplace approximation \citep{rue2007approximate}, expectation propagation \citep{cseke2011approximate}, and variational inference \citep{hensman_variational_2018}.
A complementary line of work develops specialised inference procedures or kernel constructions for particular forms of information, including inequality and shape constraints \citep{riihimaki_gaussian_2010, lopez-lopera_finite-dimensional_2018}, differential equations and physical constraints \citep{hamelijnck2024physics,long_autoip_2022}, and non-Gaussian likelihoods \citep{nickisch_state_2018}.
More recently, flow-based approaches have provided a general mechanism for conditioning GPs on arbitrary differentiable, non-linear and non-Gaussian information \citep{moss_conditioning_2026, sharrock_latentflow_2026}.

\textbf{Hyperparameter optimisation. }Kernel hyperparameters in standard GP regression are commonly learned by maximising the marginal likelihood \citep{rasmussen_gaussian_2006}, and analogous approximate objectives exist for non-conjugate likelihoods \citep{li_improving_2023}.
When conditioning information is specified through general constraints rather than a conventional pointwise likelihood, standard marginal-likelihood hyperparameter learning is no longer directly applicable, motivating constraint-specific likelihoods or inference procedures ~\citep{lopez-lopera_finite-dimensional_2018}.
While versatile, the GP flows introduced by \cite{moss_conditioning_2026} are computationally expensive and afford no practical way to optimise kernel hyperparameters.

\section{GAUSSIAN PROCESSES AND SAMPLING}
\label{sec: gps-sampling}

Let $f\colon\mathcal{X}\to\mathbb{R}$ be a random function over $\mathcal{X}\subseteq\mathbb{R}^{D}$. 
Assume $f$ has a GP prior, written $f \sim \mathcal{GP}(\mu,k_{\boldsymbol\theta})$, specified by a mean function $\mu\colon\mathcal{X}\to\mathbb{R}$ and a positive semidefinite kernel $k_{\boldsymbol\theta}\colon\mathcal{X}\times\mathcal{X}\to\mathbb{R}$ with hyperparameters $\boldsymbol\theta$. 
By definition, for every finite collection of inputs $X=(x_1,\ldots,x_N)$, the vector of GP function evaluations $\mathbf f \coloneqq [f(x_1),\ \ldots,\ f(x_N)]^{\top}$
is multivariate Gaussian distributed $\mathbf f \sim \mathcal{N}(\mathbf m,\mathbf K(\boldsymbol\theta)),$ with elements $\mathbf m_i=\mu(x_i)$ and $\mathbf K_{ij}(\boldsymbol\theta)=k_{\boldsymbol\theta}(x_i,x_{j})$ for $i, j \in \{1, \dots, N\}$.
Commonly, we sample from $\mathbf f \equiv f(X)$ by the mean-scale transformation $\mathbf f = \mathbf m+\mathbf L\mathbf z$, where $\mathbf L\mathbf L^\top=\mathbf K$ and white noise $\mathbf z\sim\mathcal{N}(\mathbf 0,\mathbf I_N)$. 
In general, constructing $\mathbf L$ costs $\mathcal{O}(N^3)$ time and $\mathcal{O}(N^2)$ memory, and the subsequent sample costs $\mathcal{O}(N^2)$. 
%

When the covariance function is stationary, there are well-studied solutions to scalable sampling of GPs.
Through application of Bochner's theorem, a number of finite-feature kernel approximations have been proposed, including random Fourier features \citep{rahimi_random_2007} and Laplacian eigenfunctions \citep{solin_know_2019,solin_hilbert_2020}. 
In each case, the full kernel matrix is replaced with a low-rank approximation, defined by the number of features used $R$, and samples are drawn by projecting standard normal samples in the coefficient space back into the function space.
These methods result in a
sampling time complexity of $\mathcal{O}(NR)$, with $R$ serving as a tunable parameter that determines approximation quality.

When $k_{\boldsymbol\theta}$ is stationary, we approximate the kernel using $R$ basis functions $\{\psi_{\boldsymbol\theta}^{r}\}_{r=1}^{R}$ with feature vector $\boldsymbol\psi_{\boldsymbol\theta}(x) = [\psi_{\boldsymbol\theta}^{1}(x),\ \ldots,\ \psi_{\boldsymbol\theta}^{R}(x)]^\top$ and matrix $\boldsymbol\Psi_{\boldsymbol\theta} = \boldsymbol\psi_{\boldsymbol\theta}(X)^\top \in\mathbb{R}^{N\times R}$.
For a suitable choice of basis representation (cf. App. \ref{app: bases} for examples), the kernel approximation takes the form $
    k_{\boldsymbol\theta}(x,x')
    \approx
    \boldsymbol\psi_{\boldsymbol\theta}(x)^\top\boldsymbol\psi_{\boldsymbol\theta}(x'),
    $
    so that
    $
    \mathbf K(\boldsymbol\theta)
    \approx
    \widetilde{\mathbf K}(\boldsymbol\theta)
    =
    \boldsymbol\Psi_{\boldsymbol\theta}\boldsymbol{\Psi_{\boldsymbol\theta}}^\top.
    \label{eq: feature-kernel}
$
Equivalently, the approximate GP may be represented directly in weight space using the decoder function $g_{\boldsymbol\theta}$, namely, 
\begin{equation}
    g_{\boldsymbol\theta} (\boldsymbol\xi) \coloneqq 
    \widetilde f_{\boldsymbol\theta}
    =
    \mu+\boldsymbol\psi_{\boldsymbol\theta}^\top \boldsymbol\xi,
    \qquad
    \boldsymbol\xi\sim\mathcal{N}(\mathbf 0,\mathbf I_R).
    \label{eq: weight-space-gp}
\end{equation}
A functional sample can be evaluated at arbitrary locations, incurring only an $\mathcal{O}(NR)$ cost \citep{rahimi_random_2007, solin_know_2019, solin_hilbert_2020}. 
Here, $\boldsymbol\xi \in \mathbb{R}^R$ with $R \ll N$, and \eqref{eq: weight-space-gp} provides the primary mechanism by which we reduce computation in the developments to come.
Using a finite-feature representation is useful beyond scalability; by intelligently selecting a basis in which to approximate the kernel, problem-specific structure may be enforced exactly.

When $k_{\boldsymbol\theta}$ is non-stationary, scalable sampling is more difficult as the basis decomposition in \eqref{eq: weight-space-gp} does not hold. This can arise, for instance, when $k_{\boldsymbol\theta}$ corresponds to the kernel of a conditioned GP.
\cite{wilson_efficiently_2020} describe a decoupled basis that provides exact pathwise corrections to approximate prior samples, 
 thus approximating the GP posterior.
%
In Section \ref{sec:update-not-diffuse}, we will adapt this framework to pre-condition linear-Gaussian information into the GP flows.

\subsection{Non-linear and Non-Gaussian Conditioning}

Finite-feature representations do not resolve the difficulties of non-conjugacy. 
Recent flow-based methods, introduced for GPs by \citet{moss_conditioning_2026} and since extended to arbitrary stochastic processes by \citet{sharrock_latentflow_2026}, provide a general framework for this task, by reframing conditional sampling as a transport problem.
Consider the interpolation
\begin{equation} \label{eqn:gauss_interp}
    \mathbf f_t
    =
    \alpha(t)\mathbf f_0
    +
    \sqrt{1-\alpha(t)^2}\,\mathbf z,
    \qquad
    \mathbf z\sim\mathcal{N}(\mathbf 0,\mathbf I_N),
\end{equation}
where $\alpha(t) = \exp (-\frac{1}{2} \int_0^t \beta(s) \mathrm{d}s)$ for a typical noise schedule $\beta(t)$.
As $t$ increases, this progressively corrupts the sample $\mathbf f_0$ towards Gaussian noise.
%
Since $\mathbf f_0\sim\mathcal N(\mathbf m,\mathbf K(\boldsymbol\theta))$, the marginals of this interpolation are given in closed form,
\begin{equation} \label{eqn:marginals}
    p_t(\mathbf f_t)
    =
    \mathcal{N}\!\left(
        \alpha(t)\mathbf m,\,
        \alpha(t)^2\mathbf K(\boldsymbol\theta)+
        \bigl(1-\alpha(t)^2\bigr)\mathbf I_N
    \right),
\end{equation}
defining a path, continuous in $t$, from the GP distribution to a standard Gaussian reference.
Equivalently, by transporting samples along the path given by the reverse-time stochastic differential equation (SDE) 
\begin{equation} 
    \label{eq: reverse}
    \mathrm{d}\mathbf{f}_t = -\beta(t)\big[
    \tfrac{1}{2} \mathbf{f}_t
    + \nabla_{\mathbf{f}_t}
    \log p_t(\mathbf{f}_t)
    \big]\mathrm{d}t
    +\sqrt{\beta(t)}\, \mathrm{d}\overline{\mathbf{W}}_t,
\end{equation}
where $t: 1 \rightarrow 0$ and $\overline{\mathbf{W}}_t$ is a reverse-time Wiener process \citep{anderson1982reverse}, we preserve the time-$t$ marginals in \eqref{eqn:marginals}.
Because $p_t$ is Gaussian, the unconditional score $\nabla_{\mathbf f_t}\log p_t(\mathbf f_t)$ is available analytically.
Conditioning can now be incorporated by augmenting Eq. \eqref{eq: reverse} with a likelihood-dependent guidance term; for information $\mathcal C$ and by Bayes' rule
\begin{equation}
    \nabla_{\mathbf f_t}\log p_t(\mathbf f_t\mid\mathcal C)
    =
    \underbrace{
        \nabla_{\mathbf f_t}\log p_t(\mathbf f_t)
    }_{\text{closed-form}}
    +
    \underbrace{
        \nabla_{\mathbf f_t}\log p_t(\mathcal C\mid\mathbf f_t).
    }_{\text{conditional guidance}}
    \label{eq: score-decomp}
\end{equation}
The second term is calculated using samples from $p(\mathbf f_0\mid\mathbf f_t)$, available in closed-form, and computing  pointwise evaluation of the likelihood $p(\mathcal C\mid\mathbf f_0)$. 
This allows the sampler to accommodate non-linear and non-Gaussian observations and conditioning statements.

The cubic scaling of GP flows restricts utility, making repeated sampling infeasible for many important large-scale problems, limiting the ability to tune kernel hyperparameters, which, we demonstrate, is crucial for generating high-quality samples.
Some problem classes, including sampling PDE solutions, are currently more challenging than necessary, due to the error resulting from finite-difference approximations of derivatives.
Introducing a scalable kernel approximation addresses both of these problems directly.

\section{FEATUREFLOW}
\label{sec: methods}

We now re-derive the generating SDE in \eqref{eq: reverse} in the coefficient space of a finite GP expansion described in \eqref{eq: weight-space-gp}.
This separates the dimension of the sampling dynamics from the resolution of the evaluation grid, making repeated computation possible. 
In turn, this enables scalable estimation of kernel hyperparameters. 
Gaussian observations are incorporated exactly through a Matheron-corrected decoder \citep{wilson_efficiently_2020}, simplifying the path-dynamics of the generative flow.
%
Finally, we introduce methodology to adapt the kernel hyperparameters during integration of \eqref{eq: reverse} using the same conditional information that guides the flow, resulting in a parsimonious framework for joint estimation of the process $f$ and its hyperparameters. 

\begin{figure}
    \centering
    \includegraphics[width=0.8\linewidth]{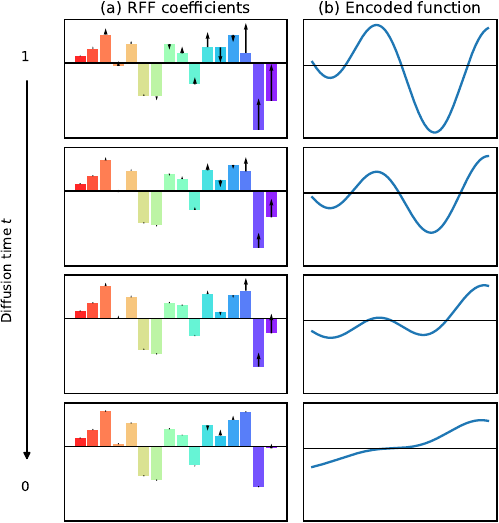}
    \caption{Schematic of \textsc{FeatureFlow} with fixed hyperparameters. RFF coefficients \textbf{(a)} and the functions they encode \textbf{(b)} for a single sampling trajectory backwards in diffusion time, for a monotone increasing condition.}
    \label{fig: schematic}
\end{figure}

\subsection{Conditional Flow in Weight Space}
\label{subsec:weight_flow}

Assume, here, the kernel hyperparameters $\boldsymbol\theta$ are fixed; their adaptation is introduced in Sec.~\ref{sec: methods-hyperparameters}.
Assuming an appropriate kernel decomposition as in \eqref{eq: weight-space-gp}, such that $f \approx g_{\boldsymbol\theta}(\boldsymbol \xi)$, we apply the same distributional interpolation as in \eqref{eqn:gauss_interp} directly to the weight vector
\begin{equation}
    \boldsymbol \xi_t
    =\alpha(t)\boldsymbol \xi_0
    +\sqrt{1-\alpha(t)^2}\,\mathbf z,
    \quad
    \boldsymbol\xi_0,\mathbf z\sim\mathcal N(\mathbf0,\mathbf I_R).
    \label{eq:methods-weight-interpolation}
\end{equation}
See Fig.~\ref{fig: schematic}(a) for a schematic. For fixed $\boldsymbol\theta$, define the decoded likelihood by $L_{\boldsymbol\theta}(\boldsymbol\xi)
    \coloneqq
    p\left(\mathcal C\mid g_{\boldsymbol\theta}(\boldsymbol\xi)\right)$.
%
In the weight-space, the target posterior is given by
\begin{equation}
    \pi_{\boldsymbol\theta}(\boldsymbol\xi_0\mid\mathcal C)
    =
    \frac{
        \mathcal N(\boldsymbol\xi_0;\mathbf0,\mathbf I_R)
        L_{\boldsymbol\theta}(\boldsymbol\xi_0)
    }{Z_{\boldsymbol\theta}},
    \label{eq:methods-weight-target}
\end{equation}
with $Z_{\boldsymbol\theta}=\mathbb E_{\boldsymbol\xi\sim\mathcal N(\mathbf0,\mathbf I_R)}\left[L_{\boldsymbol\theta}(\boldsymbol\xi)\right]$.
A posterior draw from $\widetilde f_{\boldsymbol\theta}$ is obtained by propagating a draw of $\boldsymbol\xi_0$ from \eqref{eq:methods-weight-target} through \eqref{eq: weight-space-gp}.
%
Thus an exact conditional score and exact integration recover the posterior induced by the chosen finite-feature GP.
As both $\boldsymbol\xi_0$ and the noising variable are standard normal, the distribution of $\boldsymbol\xi_t$ in \eqref{eq:methods-weight-interpolation} is $\mathcal{N}(\mathbf0,\mathbf I_R)$ for all $t$.


Following \eqref{eq: reverse}, 
the reverse-time SDE that corresponds to \eqref{eq:methods-weight-interpolation} is
\begin{equation}
\begin{aligned}
    \mathrm d\boldsymbol\xi_t
    =-\beta(t)\left[
        \tfrac12\boldsymbol\xi_t
        +\nabla_{\boldsymbol\xi_t}\log p_t(\boldsymbol\xi_t)
    \right]\mathrm dt
    +\sqrt{\beta(t)}\,\mathrm d\overline{\mathbf W}_t
    \label{eq:methods-weight-reverse-sde}
\end{aligned}
\end{equation}
with dimension $R\ll N$.
Define the smoothed likelihood
\begin{equation}
    \label{eq: smooth-like}
    \overline L_{t,\boldsymbol\theta}(\boldsymbol\xi_t)
    \coloneqq
    \mathbb E_{\boldsymbol\xi_0\sim p(\boldsymbol\xi_0\mid\boldsymbol\xi_t)}
    \!\left[L_{\boldsymbol\theta}(\boldsymbol\xi_0)\right].
\end{equation}
The conditional score then decomposes as
\begin{equation}
    \nabla_{\boldsymbol\xi_t}\log p_t(\boldsymbol\xi_t\mid\mathcal C)
    =-\boldsymbol\xi_t
    +\underbrace{
        \nabla_{\boldsymbol\xi_t}
        \log\overline L_{t,\boldsymbol\theta}(\boldsymbol\xi_t)
    }_{\mathbf h_{t,\boldsymbol\theta}(\boldsymbol\xi_t)}.
    \label{eq:methods-weight-score}
\end{equation}
The unconditional score is therefore trivial in weight space, and all problem-specific computation is confined to the guidance term $\mathbf h_{t,\boldsymbol\theta}$ such that we need only solve the reverse SDE
\begin{equation}
\begin{aligned}
    \mathrm d\boldsymbol\xi_t
    =\beta(t)\left[
        \tfrac12\boldsymbol\xi_t
        -h_{t,\boldsymbol\theta}(\boldsymbol\xi_t)
    \right]\mathrm dt
    +\sqrt{\beta(t)}\,\mathrm d\overline{\mathbf W}_t.
    \label{eq:methods-weight-reverse-sde2}
\end{aligned}
\end{equation}

\subsection{Computing Weight-space Guidance}

With the score decomposed as in Eq.~\eqref{eq:methods-weight-score}, all that remains is the guidance \(\mathbf h_{t,\boldsymbol\theta}(\boldsymbol\xi_t) = \nabla_{\boldsymbol\xi_t}\log \overline L_{t,\boldsymbol\theta}(\boldsymbol\xi_t)\).
This is the gradient of an \(R\)-dimensional integral of the decoded likelihood, which is intractable for non-linear conditions and, despite the reduced dimension, far too large for quadrature.
We therefore estimate it by Monte Carlo, using $S$ bridge samples.
Let
\begin{equation}
\label{eq:snis-quantities}
    \ell_s=\log L_{\boldsymbol\theta}(\boldsymbol\xi_0^{(s)}),
    \qquad
    \overline w_s
    =\frac{\exp(\ell_s)}{\sum_{r=1}^{S}\exp(\ell_r)}.
\end{equation}
Differentiating through Eq.~\eqref{eq: smooth-like} gives
\begin{equation}
    \widehat{\mathbf h}_{t,\boldsymbol\theta}(\boldsymbol\xi_t)
    =
    \alpha(t)
    \sum_{s=1}^{S}\overline w_s
    \nabla_{\boldsymbol\xi_0^{(s)}}
    \log L_{\boldsymbol\theta}(\boldsymbol\xi_0^{(s)}).
    \label{eq:methods-mc-guidance}
\end{equation}
Automatic differentiation provides the likelihood gradient through both the condition and, where needed, the decoder.
Equation~\eqref{eq:methods-mc-guidance} is the Monte Carlo guidance estimator used in the experiments that follow.

Guidance is evaluated in weight space: for a condition $\mathcal C$, its likelihood is computed after applying the differentiable decoder in Eq.~\eqref{eq: weight-space-gp}, as shown in Fig. \ref{fig: schematic}.
For some problems, conditions may be naturally expressed in weight-space (e.g. band-limited signals), rendering the decoder step unnecessary and yielding even greater cost reductions.
This highly scalable flow over approximate GPs via their weight-space representation is a special case of a tilt given in \citet{sharrock_latentflow_2026}.
Function values are formed on the full output grid only when required by a condition or when returning the final sample.
A condition may instead be evaluated at a smaller set of collocation points, making its cost independent of the resolution used for downstream tasks.

As the flow is performed in weight space, where the unconditional covariance structure is the identity, the major cost of sampling, $\mathcal{O}(N^3)$ from the kernel factorisation, is eliminated entirely.
Evaluating the decoder on $N_c$ locations therefore costs $\mathcal O(RN_c)$. With linear-Gaussian data, the remaining dense factorisation is only dependent on the number of observed data (cf. App. \ref{app:update}). 
The cost of adaptive hyperparameter selection is thus governed by the number of features and the number of observations, rather than by the resolution of the output grid.

\subsection{Exact Conjugate Updates by Matheron}
\label{sec:update-not-diffuse}
Suppose the conditioning information splits into a conjugate part $\mathcal{D} = \{X_y, \mathbf{y}\}$, with $\mathbf{y} = \mathbf{f}(X_y) + \sigma_y\boldsymbol\varepsilon$, and a non-conjugate remainder $\mathcal{C}$.
In this case, the weight-space target factorises as
\begin{equation}
  \pi(\boldsymbol\xi_0 \mid \mathcal{C}, \mathcal{D}) \propto
  \mathcal{N}(\boldsymbol\xi_0; \mathbf{0}, \mathbf{I}_R)\;
   L_{\boldsymbol\theta}(\boldsymbol\xi_0)
  p\big(\mathcal{D} \mid g_{\boldsymbol\theta}(\boldsymbol\xi_0)\big).
  \label{eq:joint-target}
\end{equation}
One could incorporate \(\mathcal D\) as an additional guidance term, but the resulting drift becomes numerically unstable as \(t \to 0\), with spectral radius growing as \(\beta(t)\lambda_{\max}/\sigma_y^2\) (App.~\ref{app:stiff}); for small observation noise, explicit solvers then need thousands of steps.
Instead, we absorb \(\mathcal D\) analytically into the decoder, leaving the flow to account only for \(\mathcal C\).
This resembles the whitened form of \citet{moss_conditioning_2026}, which runs the flow relative to the GP posterior given \(\mathcal D\).
However, neither version of the existing whitenings suit our setting: in function space it requires an \(\mathcal{O}(N^3)\) factorisation of the posterior covariance, and in weight space it would make the reference non-isotropic.

In \textsc{FeatureFlow}, we instead augment the latent with $|X_y|$ observation-noise variables, so that conditioning on $\mathcal D$ is carried entirely by an affine decoder via Matheron's rule \citep{wilson_efficiently_2020}. Here, we turn a prior sample into a
posterior sample by adding a correction that depends on how far the sample
is from the data. With the finite-feature prior $g_{\boldsymbol\theta} (\boldsymbol\xi)
    =
    \mu+\boldsymbol\psi_{\boldsymbol\theta}^\top \boldsymbol\xi,
$
set $\boldsymbol\Phi = \boldsymbol \Psi_{\boldsymbol \theta}$ and $\boldsymbol\Phi_y=\boldsymbol\psi_{\boldsymbol\theta}(X_y)^\top$ denoting the prior
features at the evaluation locations, a sample from
the posterior given $\mathcal{D}$ is
\begin{equation}
\nonumber
  g_{\boldsymbol\theta}^{|\mathbf{y}}(\hat{\boldsymbol\xi})
  = \underbrace{g_{\boldsymbol\theta} (\boldsymbol\xi)}_{\text{prior sample}}
  + \underbrace{\boldsymbol\Phi\boldsymbol\Phi_y^\top\mathbf{A}_{\boldsymbol\theta}
    \big(\mathbf{y} - g_{\boldsymbol\theta} (\boldsymbol\xi)(X_y)
    - \sigma_y\boldsymbol\varepsilon\big)}_{\text{conjugate update term}},
  \label{eq:matheron-decoder}
\end{equation}
where $\hat{\boldsymbol\xi} = (\boldsymbol\xi, \boldsymbol\varepsilon)\sim\mathcal{N}(\mathbf 0,\mathbf I_{R+|X_y|})$ and $\mathbf{A}_{\boldsymbol\theta} = (\boldsymbol\Phi_y\boldsymbol\Phi_y^\top + \sigma_y^2\mathbf{I})^{-1}$. 
In App.~\ref{prop:update}, we show that $g_{\boldsymbol\theta}^{|\mathbf{y}}(\hat{\boldsymbol\xi})$ is distributed
according to the exact posterior of the finite-feature GP given $\mathcal{D}$. We can then apply the flow of Sec.~\ref{subsec:weight_flow} unchanged to this affine
decoder, with the latent enlarged by one dimension per observation.


\subsection{Hyperparameter Optimisation}
\label{sec: methods-hyperparameters}
Existing implementations of GP flows \citep{moss_conditioning_2026} assume known kernel hyperparameters, or else tune those parameters manually, an approach which is rarely applicable in practice.
Hyperparameter tuning via maximum likelihood estimation is typical for standard GP regression \citep{rasmussen_gaussian_2006}; however, no equivalent procedure has been proposed for the broader class of tasks targeted by GP flows.
%
%

In our more general setting, the natural quantity to optimise is the marginal likelihood, or evidence, of the condition.
%
In particular, we target the logarithm
\begin{equation}
    \begin{aligned}
    J(\boldsymbol\theta)
    &=
    \log Z_{\boldsymbol\theta}
    =
    \log
    \mathbb E_{\boldsymbol\xi\sim\mathcal N(\mathbf0,\mathbf I_R)}
    \left[
        L_{\boldsymbol\theta}(\boldsymbol\xi)
    \right]
    \\
    &=
    \log
    \mathbb E_{\boldsymbol\xi_t\sim\mathcal N(\mathbf 0,\mathbf I_R)}
    \left[
        \overline L_{t,\boldsymbol\theta}(\boldsymbol\xi_t)
    \right],
    \end{aligned}
    \label{eq:methods-hyperparameter-objective}
\end{equation}
where the second line follows from the law of total expectation under the stationary Gaussian reference.
If the conditioning information also contains a conjugate linear-Gaussian component $\mathcal D$, 
its marginal log-likelihood can be evaluated analytically and added to this objective; see App.~\ref{sec:evidence-decomp}.

The objective in \eqref{eq:methods-hyperparameter-objective} balances condition satisfaction against the relative-entropy cost of deforming the underlying GP law into the conditioned law; see App.~\ref{subsec:static-interpretation}.
This relative-entropy cost is, up to a boundary term, precisely the energy of the guidance drift.
Thus, equivalently, the objective penalises the amount of guidance required to realise the conditioned law; see App.~\ref{subsec:dynamic-interpretation}.

A principled approach to optimising Eq.~\eqref{eq:methods-hyperparameter-objective} is to use an outer-loop scheme.
Fix an initial value $\boldsymbol\theta^{(0)}$.
Then, at outer iteration $m$, run the conditional flow with $\boldsymbol\theta^{(m)}$ fixed, construct an estimate $\widehat G(\boldsymbol\theta^{(m)})$ of the evidence gradient, and update
\begin{equation}
    \boldsymbol\theta^{(m+1)}
    =
    \boldsymbol\theta^{(m)}
    +
    \eta_m
    \widehat G(\boldsymbol\theta^{(m)}),
    \label{eq:methods-outer-parameter-update}
\end{equation}
where $\eta_m>0$ is a step size.
In practice, this update can instead be performed using a general-purpose optimiser such as Adam.

For a flow solve with the parameter fixed, the terminal particles provide a direct estimator of the evidence gradient.
Indeed, at $t=0$, Fisher's identity gives
\begin{equation}
    \nabla_{\boldsymbol\theta}J(\boldsymbol\theta)
    =
    \mathbb E_{
        \boldsymbol\xi_0
        \sim
        \pi_{\boldsymbol\theta,R}(\cdot\mid\mathcal C)}
    \left[
        \nabla_{\boldsymbol\theta}
        \log
        L_{\boldsymbol\theta}(\boldsymbol\xi_0)
    \right].
    \label{eq:methods-endpoint-evidence-gradient}
\end{equation}
Thus, for terminal particles $\{\boldsymbol\xi_0^{(b)}\}_{b=1}^{B}$, we can estimate
\begin{equation}
    \widehat G_0(\boldsymbol\theta)
    =
    \tfrac{1}{B}
    \textstyle\sum_{b=1}^{B}
    \nabla_{\boldsymbol\theta}
    \log
    L_{\boldsymbol\theta}(\boldsymbol\xi_0^{(b)})
    \approx
    \nabla_{\boldsymbol\theta}J(\boldsymbol\theta).
    \label{eq:methods-endpoint-gradient-estimator}
\end{equation}
In fact, the same gradient can also be estimated at intermediate noising times using quantities already evaluated during sampling.
For any $t\in[0,1]$, Fisher's identity gives
\begin{equation}
    \nabla_{\boldsymbol\theta} J(\boldsymbol\theta)
    =
    \mathbb E_{
        \boldsymbol\xi_t
        \sim p_{t,\boldsymbol\theta}(\boldsymbol\xi_t\mid\mathcal C)}
    \big[
        \underbrace{
            \nabla_{\boldsymbol\theta}
            \log
            \overline L_{t,\boldsymbol\theta}(\boldsymbol\xi_t)
        }_{
            \mathbf a_{t,\boldsymbol\theta}(\boldsymbol\xi_t)
        }
    \big].
    \label{eq:methods-evidence-gradient}
\end{equation}
Thus, similar to the guidance $\mathbf h_{t,\boldsymbol\theta}$, parameter estimation requires a logarithmic derivative of the smoothed likelihood.
Using the same bridge construction and normalised likelihood weights as in Sec.~\ref{subsec:weight_flow}, the corresponding Monte Carlo estimate for particle $b$ is
\begin{equation}
    \widehat{\mathbf a}_{t,\boldsymbol\theta}
    (\boldsymbol\xi_t^{(b)})
    =
    \textstyle \sum_{s=1}^{S}
    \overline w_{b,s}
    \nabla_{\boldsymbol\theta}
    \log
    L_{\boldsymbol\theta}(\boldsymbol\xi_0^{(b,s)}),
    \label{eq:methods-mc-parameter-score}
\end{equation}
where $\overline w_{b,s}$ denotes the normalised likelihood weight associated with the $s$-th bridge sample for particle $b$.
Averaging over the particles gives
\begin{equation}
    \widehat G_t(\boldsymbol\theta)
    =
    \tfrac{1}{B}
    \textstyle \sum_{b=1}^{B}
    \widehat{\mathbf a}_{t,\boldsymbol\theta}
    (\boldsymbol\xi_t^{(b)})
    \approx
    \nabla_{\boldsymbol\theta}J(\boldsymbol\theta).
    \label{eq:methods-objective-alt-alt}
\end{equation}

In practice, the outer-loop scheme can be expensive, since the conditional flow must be rerun after each parameter update.
We thus also consider a faster online scheme, in which parameter updates  using the estimate in Eq.~\eqref{eq:methods-objective-alt-alt} are interleaved with a single conditional solve.
Since changing the parameters also changes the target conditional distribution and the corresponding guidance, now the current particles need not follow the instantaneous marginal $p_{t,\boldsymbol\theta}(\boldsymbol\xi_t\mid\mathcal C)$ in Eq.~\eqref{eq:methods-evidence-gradient}, introducing an additional path-dependent approximation.
In this case, after online optimisation, the conditional flow should be rerun with the fitted value of the parameter to obtain samples from the corresponding fitted conditional model; see App.~\ref{app: hyperparam-opt} for further details.

\section{Empirical Evaluation}

\subsection{Probabilistic Downscaling}
Probabilistic downscaling of spatial fields is an increasingly important use case for generative modelling \citep{peleg_advanced_2017, deshon_stochastic_2020, perezhogin_generative_2023, hess_fast_2025}.
The fundamental task is to produce probable, high-resolution fields, given some summary statistic measured on a coarser grid, often alongside direct measurements or prior knowledge (e.g. imposed non-linear physics).

We first generate a high-resolution ground-truth field by sampling an approximate posterior GP with a large number of random Fourier features (RFFs) ($R=4096$). 
This field is then coarse-grained into a pair of $8 \times 8$ grids using the (non-linear) areal $\min$ and $\max$ summary statistics, and $32$ random locations in the domain are queried to define an additional conjugate condition.
We use an RFF basis ($R=1024$) to approximate the RBF prior, augmented with an additional Matheron's rule block to exactly correct for the Gaussian observations, for a total latent dimension of $1056$.
We then generate conditional samples using a GP flow and \textsc{FeatureFlow}, under a fixed computational budget (proxied here by a maximum wall-time of $0.5$ s).

In Figs. \ref{fig: downscale-results}(b, c), we show one sample from the GP flow and \textsc{FeatureFlow}. 
Under the same computational constraint, \textsc{FeatureFlow} is able to produce samples of comparable quality (cf. Supplementary Fig. \ref{fig: downscaling-sample-quality}) at >$10\times$ as many evaluation points as the dense method. 
This translates to a wall-time speed-up of more than two orders of magnitude at high resolutions (Fig. \ref{fig: downscale-results}(d)).
Memory requirements are also far more limited, with the vanilla GP flow estimated to require $\sim$34 GB to build the GP Gram matrix alone, compared to the equivalent \textsc{FeatureFlow} requirement of $\sim$500 MB to store the basis design matrix, at the largest resolution considered.
Further details are found in App.~\ref{app: downscaling}.
\begin{figure*}
    \centering
    \includegraphics[width=0.9\linewidth]{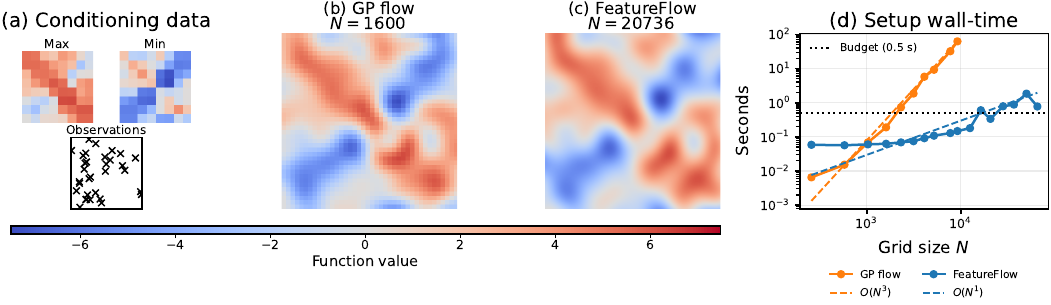}
    \vspace{-3mm}
    \caption{A non-linear downscaling task. \textbf{(a)} Non-conjugate ($\max$/$\min$) and conjugate (Gaussian observations) conditioning data. Single samples drawn from each of \textbf{(b)} the GP flow and \textsc{(c)} \textsc{FeatureFlow} under a fixed compute constraint. Both methods achieve mean absolute condition residuals $\sim 10^{-3}$. \textbf{(d)} Comparison of the two methods, showing the increase in the dominant setup wall-time with grid size $N$. The fixed time/compute budget used for \textbf{(b, c)} is indicated.}
    \label{fig: downscale-results}
    \vspace{-3mm}
\end{figure*}

\subsection{Allen-Cahn Solution Fields}

PDE solutions are commonly subject to boundary conditions (BCs), whose type and sufficiency determine the uniqueness of a solution \citep{evans10}.
In the dense GP flow framework, BCs can only be imposed inelegantly by conditioning on additional artificial noiseless data, resulting in a soft enforcement of the condition.
Here, we leverage the work of \citet{solin_know_2019} to use the Laplacian eigenfunction basis to enforce homogeneous Dirichlet BCs by construction, beyond the linear-Gaussian regime they demonstrate.

We draw samples from both a GP flow and \textsc{FeatureFlow}, conditioned such that they solve the Allen-Cahn equation:
\begin{equation}
\left\{
\begin{aligned}
f
    &= 0.2\,f^3
       - (2 \times 10^{-3})\,\Delta f,
    && \text{on } \Omega, \\
f
    &= 0,
    && \text{on } \partial\Omega,
\end{aligned}
\right.
\end{equation}
where $\Omega \subseteq \mathbb{R}^2$ is some irregular domain and $\Delta$ is the 2-D Laplacian operator.
The condition is formulated by providing a differentiable relaxation to the condition that the PDE residual should vanish everywhere.
Here, we exploit that for a stationary prior encoded by a differentiable basis, derivatives of samples may be computed analytically without the need for a finite-difference scheme as
$
\smash{
    \partial_{x_d}\widetilde f(x)
    =
    \partial_{x_d}\mu(x)
    +
    (\partial_{x_d}\boldsymbol\psi(x))^\top\boldsymbol \xi
}
$.
The Allen-Cahn residuals can therefore be evaluated analytically in $\mathcal O(RN)$ time in \textsc{FeatureFlow}.
We impose the Dirichlet boundary condition using artificial data for the GP flow, and via an appropriate Laplacian basis for the weight-space flow (App.~\ref{app: laplacian-basis}). 

In Fig. \ref{fig: pde-results}, we show results of our online adaptive sampling scheme applied to both sampling methods. 
A ground truth solution of the PDE with the given boundary conditions was computed using a numerical scheme (damped Newton iterations using exact sparse Jacobian), providing the 16 noisy ($\sigma_n = 10^{-2}$) Gaussian observations shown across the irregular domain in Fig. \ref{fig: pde-results}(a).
We show samples drawn from both samplers, where in each case, we first use deliberately misspecified kernel hyperparameters (Figs. \ref{fig: pde-results}(b, f)), and then use the adaptive scheme (Figs. \ref{fig: pde-results}(c, g)). 
We obtain higher-quality samples than those resulting from the GP flow, as measured by condition residuals (Figs. \ref{fig: pde-results}(d, h)), while benefiting from more than an order of magnitude speed-up at the highest resolutions (Figs. \ref{fig: pde-results}(e)).
In particular, the boundary conditions are adhered to much more closely using the Laplacian eigenfunction features, and the PDE residuals are generally reduced using the adaptive sampling scheme in feature space.

We hypothesise that the improved PDE consistency is due to the removal of finite-difference discretisation error when evaluating the pointwise likelihood.
Hyperparameter estimation is also more stable using \textsc{FeatureFlow} in this experiment; in Supplementary Fig. \ref{fig: pde-param-histories}, we see substantial disagreement between runs with different resolutions with the vanilla GP flow, which is minimal using our stable weight-space formulation.
Further experimental details are provided in App.~\ref{app:pde}.

\subsection{Sea Level Anomaly Estimation}
\label{sec: sla-estimation}

We now use \textsc{FeatureFlow} in a real-data example to recover sea-level anomaly (SLA) fields from observational satellite data.
Accurate SLA estimates are crucial for understanding large-scale ocean surface currents, impacting transport and navigation, as well as environmental mapping \citep{yu_prediction_2026}.
We assume a ground-truth SLA field from the Integrated Marine Observing System, Australia, which is computed using a numerical model \citep{IMOS}, to generate a synthetic Level 0 dataset: completely raw data at instrument resolution. 
%
Specifically, let $f: \mathcal{X} \rightarrow \mathbb{R}$ be a field representing the SLA over some irregular domain $\mathcal{X} \subset \mathbb{R}^2$. Here, $\mathcal{X}$ corresponds to the ocean region off the southeastern Australian coast.
Our aim is to generate probable samples of $f$ given sparse, noisy, non-Gaussian and non-linear measurements from satellite observations (cf. Figs. \ref{fig: altimeter-fields}(a, e)).


Raw satellite measurements are dependent on the true SLA through a probabilistic forward model, which additionally depends on four physical variables---significant wave height $\mathrm{SWH}$, surface wind speed $U_{10}$, surface skewness $\lambda$, and reflectivity bias $\beta$---which we will treat as known in this work.
Typically a satellite will observe at $\sim$20 Hz a radar echo comprising $\sim$100 observations of range gates (Fig. \ref{fig: altimeter-fields}(e)). The expected waveform at each range gate location $k$ is given by a five-parameter Brown's function (see App.~\ref{app: forward-model} for definition) denoted $m_{jk}$ for the $j$th observation:
\begin{equation*}
    \left\{f(\mathbf{x}_j),\,\mathrm{SWH}_j,\,U_{10,j},\,\lambda_j,\,\beta_j\right\}
    \mapsto
    \left\{m_{jk}\right\}
\end{equation*}
shown in Fig. \ref{fig: altimeter-fields}(e) by the black line.
The observation noise surrounding $m_{jk}$ is Gamma-distributed as
\begin{equation*}
    Y_{jk}\mid m_{jk}
    \sim
    \operatorname{Gamma}\!\left(
    N_l,\frac{m_{jk}}{N_l}
    \right),
    \quad
    N_l=90.
\end{equation*}
See \cite{tourain_benefits_2021}, for more details.
%
%
%
A commonly-used approach to estimating SLA is Maximum Likelihood Estimator 3  \citep[MLE3;][]{brown_average_1977, rodriguez_altimetry_1988}. 
Figure \ref{fig: altimeter-fields}(a) shows a single day's worth of satellite observation tracks over southeastern Australia; in total there were 63,458 radar echoes, each comprising 104 range gates.
%

\begin{figure*}[h!]
    \centering
    \includegraphics[width=\linewidth]{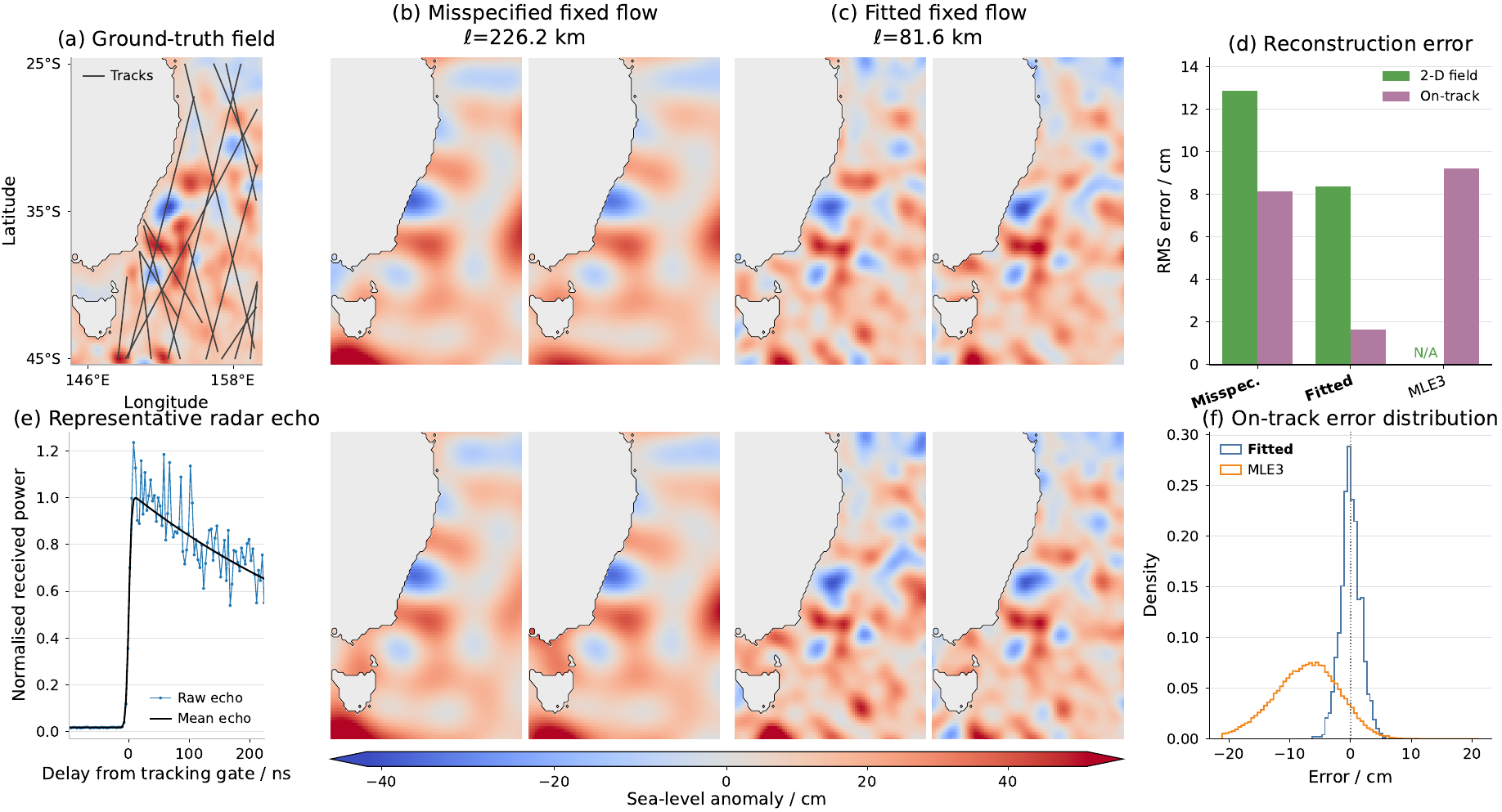}
    \vspace{-5mm}
    \caption{Inference of a sea-level anomaly field. \textbf{(a)} The ground-truth field, with satellite tracks overlaid. Samples using \textsc{FeatureFlow} with \textbf{(b)} misspecified and \textbf{(c)} fitted kernel length scales $\ell$. \textbf{(d)} The reconstruction error compared between \textsc{FeatureFlow} and the baseline. MLE3 only provides pointwise estimates and so cannot be compared across the 2-D field. \textbf{(e)} A representative radar echo, taken at 20 Hz along the satellite tracks. \textbf{(f)} Comparison of on-track reconstruction error between the deterministic baseline and scalable flow-based inference.}
    \label{fig: altimeter-fields}
    \vspace{-5mm}
\end{figure*}

Let $j=1,\ldots,J$ index the $J=63{,}458$ radar echoes, each at location $\mathbf{x}_j \in \mathcal{X}$, and $k=1,\ldots,K$ index the $K=104$ range gates. The complete observation is
$
\smash{
    \{Y_{jk}\}_{j=1:J, k=1:K}
    \in \mathbb{R}^{63{,}458\times104}.
}$
%
%
We aim to draw SLA field samples consistent with the noisy radar data, as well as enforcing an approximate RBF prior on the field through an RFF kernel approximation.
In the following results, we use $R=512$ RFFs and $100$ steps in the flow time discretisation.
The entire evaluation domain is $69{,}081$ points, a scale which is computationally infeasible using a GP flow.
We consider the adaptive setting, where the length scale is unknown \textit{a priori} and must therefore be learned from the non-Gaussian condition.
In Fig. \ref{fig: altimeter-fields}(a, b), we display the ground truth SLA field, as well as samples drawn from \textsc{FeatureFlow} with a misspecified kernel length scale. 
Figure \ref{fig: altimeter-fields}(c) shows the result of adaptive hyperparameter estimation on this real-world problem, where the complex non-Gaussian observations prove highly informative and result in visually higher-quality samples of the entire 2-D field.

Measurements of on-track and whole-field RMSE (Fig. \ref{fig: altimeter-fields}(d)) confirm that making joint predictions of SLA across all points in the domain, given all satellite measurements, provides a meaningful advantage in pointwise accuracy.
Figure \ref{fig: altimeter-fields}(f) compares the on-track error distributions between the fitted \textsc{FeatureFlow} and deterministic pointwise baseline, MLE3. 
We note that our estimates of SLA are substantially less biased than MLE3.
As previous work has aimed to reduce MLE3-related bias \citep{gomez-enri_measuring_2007, passaro_improving_2018}, in Supplementary Fig. \ref{fig: altimeter-debiased-errors}, we show that even with 
\textit{post hoc} debiasing using the ground truth field, \textsc{FeatureFlow} performs better in pointwise accuracy. 
Further experimental details are located in App.~\ref{app:sla-estimation}.
%


\section{Conclusion}
\label{sec:conclusion}

We have introduced \textsc{FeatureFlow}, which runs GP flows in the weight space of a finite-feature approximation.
This decouples the cost of the flow from the evaluation resolution and absorbs conjugate data exactly via Matheron's rule.
Because the bridge is independent of the kernel hyperparameters, the guidance samples also yield evidence gradients, enabling the first hyperparameter optimisation for GP flows.
\textsc{FeatureFlow} gives one to two orders of magnitude speed-up over function-space GP flows and scales to a 69,081-point sea-level field inferred from over six million radar observations.
Limitations include the reliance on a largely stationary prior with an accurate finite basis and the bias of the online adaptation scheme.
Future work includes jointly estimating the altimetry nuisance fields and characterising this bias.


\clearpage
\subsection*{AI Use Statement}

In this work, we used generative AI tools for providing feedback on research methodology or experiments, implementing methods, and generating synthetic data sets. We have not used generative AI tools for helping to develop theoretical models or conceptual frameworks; formulating mathematical claims; providing critical ingredients for proving mathematical claims; assisting in the writing of proofs; proposing or refining hypotheses; assisting with translation; cleaning and reformatting datasets; supporting qualitative and thematic data analysis; interpreting results. 
Additionally, we used generative AI tools for creating scientific figures, creating or editing software code, brainstorming, and sourcing/searching for information. 
We have reviewed all AI-assisted work. 
LLM-generated code was verified and tested for correctness.
We take responsibility for the final content of this work, including text, claims, or artifacts produced with the aid of generative AI.



\bibliographystyle{plainnat}
\bibliography{references}

\clearpage
\appendix
\thispagestyle{empty}
\onecolumn
\aistatstitle{Feature Space Adaptation for Effortless Gaussian Process Flows: \\
Supplementary Materials}

\section{Basis Decompositions}
\label{app: bases}
\subsection{Random Fourier features}
\label{app: rffs}

Random Fourier features (RFFs) provide a finite-dimensional approximation to stationary kernels \citep{rahimi_random_2007}. 
By Bochner's theorem, a stationary kernel $k_{\boldsymbol\theta} (x,x') = k_{\boldsymbol\theta} (x-x')$ can be written as an expectation over its spectral density $p_{\boldsymbol\theta}(\omega)$. 
Drawing frequencies $\omega_r\sim p_{\boldsymbol\theta}$ and phases $b_r\sim\operatorname{Uniform}(0,2\pi)$, we define
\begin{equation}
    \psi_{\boldsymbol\theta}^{(r)}(x)=\sqrt{\frac{2\sigma_f^2}{R}}\cos\left(\omega_r x+b_r\right),\qquad r=1,\ldots,R
\end{equation}
with $R$ defined as the number of features used to approximate the kernel.
The resulting feature vector satisfies
\begin{equation}
    k_{\boldsymbol\theta}(x,x') \approx \boldsymbol\psi_{\boldsymbol\theta}(x)^\top \boldsymbol\psi_{\boldsymbol\theta}(x'),
\end{equation}
and an approximate GP prior draw is
\begin{equation}
    \widetilde f_{\boldsymbol\theta}(x)=\mu(x) + \boldsymbol\psi_{\boldsymbol\theta}(x)^\top\boldsymbol \xi,\qquad\boldsymbol \xi \sim \mathcal{N} (\mathbf 0,\mathbf I_{R}).
\end{equation}

For the RBF kernel with length scale matrix $\boldsymbol\Lambda_{\boldsymbol\theta}$, the spectral density is Gaussian,
\begin{equation}
    \omega_r \sim \mathcal{N}\left(\mathbf 0, \boldsymbol\Lambda_{\boldsymbol\theta}^{-1}\right).
\end{equation}

Fixing the underlying standard-normal frequency draws while varying $\boldsymbol\theta$ gives a differentiable parameterisation of the features. 
Spatial derivatives are also available analytically by differentiating the trigonometric basis. 
Although RFFs directly approximate only the stationary prior kernel, they can be combined with exact kernel evaluations in the decoupled posterior construction of Section ~\ref{app:update}.

\subsection{Laplacian eigenfunction basis}
\label{app: laplacian-basis}
Laplacian eigenfunctions provide a deterministic low-rank approximation to stationary kernels while allowing boundary conditions to be imposed on a bounded domain $\Omega\subset\mathbb R^d$ \citep{solin_know_2019}. 
For homogeneous Dirichlet boundary conditions, the basis functions are obtained from the eigenvalue problem
\begin{equation}
    -\Delta\varphi_j(x)=\lambda_j^2 \varphi_j(x),
    \qquad x\in\Omega,
    \qquad
    \varphi_j(x)=0,
    \qquad x\in\partial\Omega,
\end{equation}
where the eigenfunctions are orthonormal in $L^2(\Omega)$ and ordered by increasing eigenvalue.

Let $s_{\boldsymbol\theta}(\boldsymbol\omega)$ denote the spectral density of the stationary kernel. 
For an isotropic kernel, its value depends only on the radial frequency, and the covariance function can be approximated using the first $R$ eigenfunctions as
\begin{equation}
    k_{\boldsymbol\theta}(x,x')
    \approx
    \sum_{j=1}^{R_f}
    s_{\boldsymbol\theta}(\lambda_j)
    \varphi_j(x)\varphi_j(x').
\end{equation}
Defining the weighted features
\begin{equation}
    \psi_{\boldsymbol\theta}^{(j)}(x)
    =
    \sqrt{s_{\boldsymbol\theta}(\lambda_j)}\,\varphi_j(x),
    \qquad j=1,\ldots,R,
\end{equation}
gives
\begin{equation}
    k_{\boldsymbol\theta}(x,x')
    \approx
    \boldsymbol\psi_{\boldsymbol\theta}(x)^\top
    \boldsymbol\psi_{\boldsymbol\theta}(x'),
\end{equation}
and the corresponding approximate prior draw is
\begin{equation}
    \widetilde f_{\boldsymbol\theta}(x)
    =
    \mu(x)
    +
    \boldsymbol\psi_{\boldsymbol\theta}(x)^\top\boldsymbol \xi,
    \qquad
    \boldsymbol \xi\sim\mathcal N(\mathbf0,\mathbf I_{R_f}).
\end{equation}

For simple domains, such as intervals and rectangles, the eigenpairs are available analytically. 
For irregular domains, they can instead be computed once by discretising the Laplacian operator on a grid and solving the resulting sparse eigenvalue problem. 
Restricting the discretisation to points inside $\Omega$ naturally imposes the boundary condition and ensures that every approximate prior draw vanishes on $\partial\Omega$. 
Other conditions, including homogeneous Neumann conditions, can be incorporated by changing the boundary condition in the eigenvalue problem.

The eigenfunctions and eigenvalues depend only on the domain and can therefore be held fixed while varying $\boldsymbol\theta$; the kernel hyperparameters enter only through the spectral weights $s_{\boldsymbol\theta}(\lambda_j)$. 
This gives a differentiable parameterisation with respect to the kernel hyperparameters, while replacing dense kernel operations by computations involving $R$ basis functions.

\section{Conditioning on Conjugate Data by Matheron's Rule}
\label{app:update}

Throughout this section, the conjugate data are $\mathcal{D} = \{X_y, \mathbf{y}\}$ with
$\mathbf{y} = f(X_y) + \sigma_y\boldsymbol\varepsilon$ and
$\boldsymbol\varepsilon \sim \mathcal{N}(\mathbf{0}, \mathbf{I}_{|X_y|})$.
We write $\boldsymbol\Phi=\psi_{\boldsymbol\theta}(X)^\top\in\mathbb R^{N\times R}$ and $\boldsymbol\Phi_y=\psi_{\boldsymbol\theta}(X_y)^\top \in\mathbb R^{|X_y|\times R}$ for the prior features at the evaluation and observation locations, and $\boldsymbol\mu_y = \mu(X_y)$.
Under the finite-feature prior, $\mathbf{f} = \boldsymbol\mu + \boldsymbol\Phi\boldsymbol\xi$ and
$\mathbf{y} = \boldsymbol\mu_y + \boldsymbol\Phi_y\boldsymbol\xi + \sigma_y\boldsymbol\varepsilon$ with $\boldsymbol\xi \sim \mathcal{N}(\mathbf{0}, \mathbf{I}_{R})$.
Finally, define
\[
  \mathbf{A}_{\boldsymbol\theta} = \big(\boldsymbol\Phi_y\boldsymbol\Phi_y^\top + \sigma_y^2\mathbf{I}\big)^{-1},
  \qquad
  \mathbf{P} = \boldsymbol\Phi_y^\top\mathbf{A}_{\boldsymbol\theta}\boldsymbol\Phi_y .
\]

\subsection{Matheron's rule targets the posterior}
\label{app:matheron}

Let $\hat{\boldsymbol\xi} = (\boldsymbol\xi, \boldsymbol\varepsilon) \in \mathbb{R}^{R + |X_y|}$, denoting the total latent dimension.
Expanding the decoder in Sec. \ref{sec:update-not-diffuse} using $g_{\boldsymbol\theta}(\boldsymbol\xi)(X_y) = \boldsymbol\mu_y + \boldsymbol\Phi_y\boldsymbol\xi$ shows that it is affine in $\hat{\boldsymbol\xi}$:
\begin{equation}
  g^{|\mathbf{y}}_{\boldsymbol\theta}(\hat{\boldsymbol\xi})
  = \boldsymbol\mu + \boldsymbol\Phi\boldsymbol\xi
    + \boldsymbol\Phi\boldsymbol\Phi_y^\top\mathbf{A}_{\boldsymbol\theta}
      \big(\mathbf{y} - \boldsymbol\mu_y - \boldsymbol\Phi_y\boldsymbol\xi - \sigma_y\boldsymbol\varepsilon\big)
  = \mathbf{m}_{|\mathbf{y}} + \boldsymbol\Psi_{|\mathbf{y}}\hat{\boldsymbol\xi},
  \label{eq:matheron-affine}
\end{equation}
where
\[
  \mathbf{m}_{|\mathbf{y}} = \boldsymbol\mu + \boldsymbol\Phi\boldsymbol\Phi_y^\top\mathbf{A}_{\boldsymbol\theta}(\mathbf{y} - \boldsymbol\mu_y),
  \qquad
  \boldsymbol\Psi_{|\mathbf{y}} = \big[\,\boldsymbol\Phi(\mathbf{I} - \mathbf{P}),\; -\sigma_y\boldsymbol\Phi\boldsymbol\Phi_y^\top\mathbf{A}_{\boldsymbol\theta}\,\big]\in\mathbb R^{N\times(R+|X_y|)}.
\]

\begin{proposition}
\label{prop:update}
Assume $\mathcal{C}$ and $\mathcal{D}$ are conditionally independent given $f$.
If $\hat{\boldsymbol\xi} \sim \mathcal{N}(\mathbf{0}, \mathbf{I}_{R + |X_y|})$, then
$g^{|\mathbf{y}}_{\boldsymbol\theta}(\hat{\boldsymbol\xi})$ is distributed according to the exact posterior of the finite-feature GP given $\mathcal{D}$.
Consequently, the latent target
\begin{equation}
  \pi^{|\mathbf{y}}(\hat{\boldsymbol\xi}_0 \mid \mathcal{C}) \;\propto\;
  \mathcal{N}(\hat{\boldsymbol\xi}_0; \mathbf{0}, \mathbf{I}_{R + |X_y|})\;
  p\big(\mathcal{C} \mid g^{|\mathbf{y}}_{\boldsymbol\theta}(\hat{\boldsymbol\xi}_0)\big)
  \label{eq:updated-target}
\end{equation}
pushes forward under $g^{|\mathbf{y}}_{\boldsymbol\theta}$ to the same law over functions as Eq.~\eqref{eq:joint-target} does under $g_{\boldsymbol\theta}$.
\end{proposition}

\begin{proof}
\emph{Posterior.}
Since $g^{|\mathbf{y}}_{\boldsymbol\theta}$ is affine, $g^{|\mathbf{y}}_{\boldsymbol\theta}(\hat{\boldsymbol\xi})$ is Gaussian with mean $\mathbf{m}_{|\mathbf{y}}$ and covariance $\boldsymbol\Psi_{|\mathbf{y}}\boldsymbol\Psi_{|\mathbf{y}}^\top$; it suffices to match these to the finite-feature posterior.
Writing $\widetilde{\mathbf{K}} = \boldsymbol\Phi\boldsymbol\Phi^\top$, $\widetilde{\mathbf{K}}_{\cdot y} = \boldsymbol\Phi\boldsymbol\Phi_y^\top$ and $\widetilde{\mathbf{K}}_{yy} = \boldsymbol\Phi_y\boldsymbol\Phi_y^\top$, the mean is
$\mathbf{m}_{|\mathbf{y}} = \boldsymbol\mu + \widetilde{\mathbf{K}}_{\cdot y}(\widetilde{\mathbf{K}}_{yy} + \sigma_y^2\mathbf{I})^{-1}(\mathbf{y} - \boldsymbol\mu_y)$, the posterior mean.
For the covariance, the identity $\boldsymbol\Phi_y\boldsymbol\Phi_y^\top = \mathbf{A}_{\boldsymbol\theta}^{-1} - \sigma_y^2\mathbf{I}$ gives
\[
  \mathbf{P}^2
  = \boldsymbol\Phi_y^\top\mathbf{A}_{\boldsymbol\theta}\big(\mathbf{A}_{\boldsymbol\theta}^{-1} - \sigma_y^2\mathbf{I}\big)\mathbf{A}_{\boldsymbol\theta}\boldsymbol\Phi_y
  = \mathbf{P} - \sigma_y^2\boldsymbol\Phi_y^\top\mathbf{A}_{\boldsymbol\theta}^2\boldsymbol\Phi_y ,
\]
so that $(\mathbf{I} - \mathbf{P})^2 + \sigma_y^2\boldsymbol\Phi_y^\top\mathbf{A}_{\boldsymbol\theta}^2\boldsymbol\Phi_y = \mathbf{I} - \mathbf{P}$.
Hence
\[
  \boldsymbol\Psi_{|\mathbf{y}}\boldsymbol\Psi_{|\mathbf{y}}^\top
  = \boldsymbol\Phi\big[(\mathbf{I} - \mathbf{P})^2 + \sigma_y^2\boldsymbol\Phi_y^\top\mathbf{A}_{\boldsymbol\theta}^2\boldsymbol\Phi_y\big]\boldsymbol\Phi^\top
  = \boldsymbol\Phi(\mathbf{I} - \mathbf{P})\boldsymbol\Phi^\top
  = \widetilde{\mathbf{K}} - \widetilde{\mathbf{K}}_{\cdot y}\big(\widetilde{\mathbf{K}}_{yy} + \sigma_y^2\mathbf{I}\big)^{-1}\widetilde{\mathbf{K}}_{y\cdot},
\]
which is the posterior covariance.

\emph{Pushforward.}
Let $\nu = \mathcal{N}(\mathbf{0}, \mathbf{I})$ and let $\varphi$ be any bounded measurable test function.
Under Eq.~\eqref{eq:joint-target}, and using the conditional independence of $\mathcal{C}$ and $\mathcal{D}$ given $f$,
\[
  \mathbb{E}\big[\varphi(g_{\boldsymbol\theta}(\boldsymbol\xi_0))\big]
  \;\propto\; \mathbb{E}_{\mathbf{f}\sim p_{\boldsymbol\theta}}\big[\varphi(\mathbf{f})\,p(\mathcal{D}\mid\mathbf{f})\,p(\mathcal{C}\mid\mathbf{f})\big]
  \;\propto\; \mathbb{E}_{\mathbf{f}\sim p_{\boldsymbol\theta}(\cdot\mid\mathcal{D})}\big[\varphi(\mathbf{f})\,p(\mathcal{C}\mid\mathbf{f})\big].
\]
Under Eq.~\eqref{eq:updated-target}, by the first part,
\[
  \mathbb{E}\big[\varphi(g^{|\mathbf{y}}_{\boldsymbol\theta}(\hat{\boldsymbol\xi}_0))\big]
  \;\propto\; \mathbb{E}_{\hat{\boldsymbol\xi}\sim\nu}\big[\varphi(g^{|\mathbf{y}}_{\boldsymbol\theta}(\hat{\boldsymbol\xi}))\,p(\mathcal{C}\mid g^{|\mathbf{y}}_{\boldsymbol\theta}(\hat{\boldsymbol\xi}))\big]
  = \mathbb{E}_{\mathbf{f}\sim p_{\boldsymbol\theta}(\cdot\mid\mathcal{D})}\big[\varphi(\mathbf{f})\,p(\mathcal{C}\mid\mathbf{f})\big].
\]
Both expressions are proportional to the same functional of $\varphi$, and both are normalised, so the two laws coincide.
\end{proof}

\subsection{Diffusing under conjugacy is stiff}
\label{app:stiff}
Suppose that we instead choose to treat $\mathcal{D}$ as an additional guidance term. Then the smoothed data likelihood $\nabla_{\boldsymbol\xi_t}\log p_t(\mathbf{y}\mid\boldsymbol\xi_t)$ is therefore Gaussian,
\[
  p_t(\mathbf{y}\mid\boldsymbol\xi_t)
  = \mathcal{N}\big(\mathbf{y};\ \boldsymbol\mu_y + \alpha\boldsymbol\Phi_y\boldsymbol\xi_t,\ \mathbf{M}_t\big),
  \qquad
  \mathbf{M}_t = (1-\alpha^2)\boldsymbol\Phi_y\boldsymbol\Phi_y^\top + \sigma_y^2\mathbf{I},
\]
and its score is available in closed form, with Hessian
\[
  \mathbf{H}_t
  = \nabla^2_{\boldsymbol\xi_t}\log p_t(\mathbf{y}\mid\boldsymbol\xi_t)
  = -\alpha^2\boldsymbol\Phi_y^\top\mathbf{M}_t^{-1}\boldsymbol\Phi_y ,
\]
which does not depend on $\boldsymbol\xi_t$.
Let $\lambda_1 \ge \lambda_2 \ge \dots$ denote the eigenvalues of $\boldsymbol\Phi_y\boldsymbol\Phi_y^\top$, which are also the nonzero eigenvalues of $\boldsymbol\Phi_y^\top\boldsymbol\Phi_y$.
Since $\mathbf{M}_t$ shares eigenvectors with $\boldsymbol\Phi_y\boldsymbol\Phi_y^\top$, the nonzero eigenvalues of $\mathbf{H}_t$ are
\[
  -\frac{\alpha^2\lambda_i}{(1-\alpha^2)\lambda_i + \sigma_y^2}
  \;\longrightarrow\; -\frac{\lambda_i}{\sigma_y^2}
  \qquad \text{as } t \to 0 .
\]

\paragraph{Step-size restriction.}
Written in forward integration time, the reverse-time drift is $-\tfrac12\beta(t)\boldsymbol\xi_t + \beta(t)\mathbf{h}_t$, whose Jacobian is $\beta(t)\big(\mathbf{H}_t - \tfrac12\mathbf{I}\big)$.
Along the $i$th data direction, the drift therefore contracts at rate
\[
  \kappa_i(t) = \beta(t)\left(\frac12 + \frac{\alpha^2\lambda_i}{(1-\alpha^2)\lambda_i + \sigma_y^2}\right)
  \;\approx\; \frac{\beta(t)\,\lambda_i}{\sigma_y^2}
  \qquad \text{near } t = 0 .
\]
An explicit solver such as Euler--Maruyama is stable on a linear contraction of rate $\kappa$ only if $\Delta t < 2/\kappa$, which near $t = 0$ requires
\begin{equation}
  \Delta t \;\lesssim\; \frac{2\sigma_y^2}{\beta(t)\,\lambda_{1}} .
  \label{eq:step-bound}
\end{equation}


\section{Hyperparameter Optimisation}
\label{app: hyperparam-opt}

\subsection{Evidence Decomposition with Conjugate Information}
\label{sec:evidence-decomp}

In Sec.~\ref{sec: methods}, we note that when the conditioning statement $\mathcal{C}$ contains some linear-Gaussian component, we can extract this part of the likelihood and compute it analytically.
Let $\mathcal D$ denote the conjugate conditioning information (e.g. noisy point observations with Gaussian likelihood), and let $\mathcal C$ denote the remaining non-conjugate condition.
The exact GP evidence decomposes as
\begin{equation}
    \log p_{\boldsymbol\theta}(\mathcal C,\mathcal D)
    =
    \log p_{\boldsymbol\theta}(\mathcal C\mid\mathcal D)
    +
    \log p_{\boldsymbol\theta}(\mathcal D).
    \label{eq:methods-objective-decomposition}
\end{equation}
Let $g_{\boldsymbol\theta}^{\mid y}$ denote the decoder obtained after conditioning on $\mathcal D$, and define the corresponding finite-feature conditional evidence by
\begin{equation}
    Z_{\boldsymbol\theta,R}^{\mathcal D}
    =
    \mathbb E_{\hat{\boldsymbol\xi}_0\sim\mathcal N(\mathbf 0,\mathbf I_{R+|X_y|})}
    \left[
        p\!\left(
            \mathcal C
            \mid
            g_{\boldsymbol\theta}^{\mid y}(\hat{\boldsymbol\xi}_0)
        \right)
    \right].
    \label{eq:methods-conditional-evidence}
\end{equation}
When $g_{\boldsymbol\theta}^{\mid y}$ generates the exact conditional GP, $Z_{\boldsymbol\theta,R}^{\mathcal D}=p_{\boldsymbol\theta}(\mathcal C\mid\mathcal D)$; for the finite-feature construction, it is the corresponding approximation.
We therefore optimise
\begin{equation}
    J_{\mathcal D}(\boldsymbol\theta)
    =
    \log Z_{\boldsymbol\theta,R}^{\mathcal D}
    +
    \log p_{\boldsymbol\theta}(\mathcal D).
    \label{eq:methods-objective-conjugate}
\end{equation}
For a fixed value of $\boldsymbol\theta$, the gradient of the first term in Eq.~\eqref{eq:methods-objective-conjugate} can be estimated using the same bridge construction used to guide the conditional flow.
In particular,
\begin{align}
    \nabla_{\boldsymbol\theta}J_{\mathcal D}(\boldsymbol\theta)
    &=
    \mathbb E_{
        \hat{\boldsymbol\xi}_t
        \sim
        p_{t,\boldsymbol\theta}(\hat{\boldsymbol\xi}_t\mid\mathcal C,\mathcal D)}
    \left[
        \nabla_{\boldsymbol\theta}
        \log
        \mathbb E_{
            \hat{\boldsymbol\xi}_0
            \sim
            p(\hat{\boldsymbol\xi}_0\mid\hat{\boldsymbol\xi}_t)}
        \left[
            p\!\left(
                \mathcal C
                \mid
                g_{\boldsymbol\theta}^{\mid y}(\hat{\boldsymbol\xi}_0)
            \right)
        \right]
    \right]
    +
    \nabla_{\boldsymbol\theta}
    \log p_{\boldsymbol\theta}(\mathcal D).
    \label{eq:methods-objective-gradient}
\end{align}
For outer particles $\hat{\boldsymbol\xi}_t^{(b)}$ and bridge samples $\hat{\boldsymbol\xi}_0^{(b,s)}$ constructed as in Eq.~\eqref{eq:methods-weight-interpolation}, we therefore compute
\begin{equation}
    \widehat G_t^{\mathcal D}(\boldsymbol\theta)
    =
    \frac{1}{B}\sum_{b=1}^{B}
    \nabla_{\boldsymbol\theta}
    \log\left[
        \frac{1}{S}\sum_{s=1}^{S}
        p\!\left(
            \mathcal C
            \mid
            g_{\boldsymbol\theta}^{\mathcal D}(\boldsymbol\xi_0^{(b,s)})
        \right)
    \right]
    +
    \nabla_{\boldsymbol\theta}
    \log p_{\boldsymbol\theta}(\mathcal D) = \widehat G_t(\boldsymbol\theta)
    +
    \nabla_{\boldsymbol\theta}
    \log p_{\boldsymbol\theta}(\mathcal D)
    \label{eq:methods-objective-alt}
\end{equation}
where $\widehat G_t(\boldsymbol\theta)$ denotes the Monte Carlo estimate of the non-conjugate contribution, as defined in Sec.~\ref{sec: methods-hyperparameters}.
For the particular case in which $\mathcal D=\{X_y,\mathbf y\}$ consists of pointwise observations with independent Gaussian noise,
\begin{equation}
    y_i=f(x_i)+\varepsilon_i,
    \qquad
    \varepsilon_i\sim\mathcal N(0,\sigma_y^2),
\end{equation}
the observations have the marginal distribution
\begin{equation}
    \mathbf y\mid\boldsymbol\theta
    \sim
    \mathcal N\left(
        \boldsymbol\mu_y,
        \mathbf K(\boldsymbol\theta)
        +
        \sigma_y^2\mathbf I
    \right),
\end{equation}
where
\begin{equation}
    [\boldsymbol\mu_y]_i
    =
    \mu(x_i),
    \qquad
    [\mathbf K(\boldsymbol\theta)]_{ij}
    =
    k_{\boldsymbol\theta}(x_i,x_j),
    \qquad
    x_i,x_j\in X_y.
\end{equation}
The corresponding marginal log-likelihood is therefore
\begin{equation}
\begin{aligned}
    \log p_{\boldsymbol\theta}(\mathcal D)
    =&
    -\frac{1}{2}
    (\mathbf y-\boldsymbol\mu_y)^\top
    \left(
        \mathbf K(\boldsymbol\theta)
        +
        \sigma_y^2\mathbf I
    \right)^{-1}
    (\mathbf y-\boldsymbol\mu_y)
    \\
    &-\frac{1}{2}
    \log\det\left(
        \mathbf K(\boldsymbol\theta)
        +
        \sigma_y^2\mathbf I
    \right)
    -
    \frac{|X_y|}{2}\log(2\pi).
\end{aligned}
\label{eq:methods-gaussian-mll}
\end{equation}
Consequently, this contribution can be evaluated exactly and differentiated with respect to the kernel hyperparameters without introducing additional MC error.

Looking at Eq.~\eqref{eq:methods-gaussian-mll}, we note that the Cholesky decomposition of $\mathbf K(\boldsymbol\theta)+\sigma_y^2\mathbf I$ is already required by the canonical-kernel decoupled posterior construction, so the log-determinant is easily calculable.
This evidence term is important because the non-conjugate conditional term alone often admits degenerate solutions, such as reducing the output variance until a residual is small without explaining the observed data.

\subsection{Algorithm}

\begin{algorithm}[h]
  \caption{Weight-space kernel hyperparameter optimisation}
  \label{alg:basis-kernel-fit}
  \begin{algorithmic}[1]
  \Require Non-conjugate condition $\mathcal C$, Gaussian observations $\mathcal D=\{X_y,\mathbf y\}$, initial hyperparameters $\boldsymbol\theta$, particles $B$, bridge draws $S$, update interval $q$, number of steps $K$, prior-feature count $R$
  \State Set a decreasing time grid $1=t_0>t_1>\cdots>t_K=0$
  \State Sample initial flow states $\hat{\boldsymbol\xi}_{t_0}^{(b)}\sim\mathcal N(\mathbf 0,\mathbf I_{R+|X_y|})$, $b=1,\ldots,B$
  \State Draw and fix bridge-sampling noise $\boldsymbol\epsilon^{(b,s)}\sim\mathcal N(\mathbf0,\mathbf I_{R+|X_y|})$
  \State Initialise the basis, decoder, and guidance condition at the current $\boldsymbol\theta$

  \For{$k=0,\ldots,K-1$}
      \State Advance $\hat{\boldsymbol\xi}_{t_k}^{(1:B)}$ to $\hat{\boldsymbol\xi}_{t_{k+1}}^{(1:B)}$ using the guided flow

      \If{$(k+1)\bmod q=0$}
          \State Construct detached bridge samples using the fixed noise:
          \[
          \hat{\boldsymbol\xi}_0^{(b,s)}
          =
          \alpha(t_{k+1})
          \operatorname{stopgrad}\!\left(
              \hat{\boldsymbol\xi}_{t_{k+1}}^{(b)}
          \right)
          +
          \sqrt{1-\alpha(t_{k+1})^2}\,
          \boldsymbol\epsilon^{(b,s)},
          \qquad
          b=1{:}B,\;s=1{:}S.
          \]
          \State Rebuild the differentiable basis quantities at the current $\boldsymbol\theta$

          \State Compute
          \[
          \widehat G_{t_{k+1}}^{\mathcal D}(\boldsymbol\theta)
          =
          \frac{1}{B}\sum_{b=1}^{B}
          \nabla_{\boldsymbol\theta}
          \log\left[
              \frac{1}{S}\sum_{s=1}^{S}
              p\!\left(
                  \mathcal C
                  \mid
                  g_{\boldsymbol\theta}^{\mid y}(\hat{\boldsymbol\xi}_0^{(b,s)})
              \right)
          \right]
          +
          \nabla_{\boldsymbol\theta}
          \log p_{\boldsymbol\theta}(\mathcal D)
          \]

          \State Apply an Adam ascent step to $\boldsymbol\theta$ using $\widehat G_{t_{k+1}}^{\mathcal D}(\boldsymbol\theta)$

          \State Refresh the basis, decoder, and guidance condition
      \EndIf
  \EndFor

  \State Return $\widehat{\boldsymbol\theta}\gets\boldsymbol\theta$
  \end{algorithmic}
\end{algorithm}

A schematic algorithm is given in Algorithm~\ref{alg:basis-kernel-fit}.
We use the same underlying standard-normal bridge draws across updates, providing common random numbers and reducing additional MC variation in the hyperparameter updates.
Each update first constructs $S$ terminal weights from the current whitened flow state, before rebuilding the quantities required to differentiate with respect to $\boldsymbol\theta$, including the posterior Cholesky factors.
The gradient in Eq.~\eqref{eq:methods-objective-alt} is then computed, using MC samples for the non-conjugate contribution and the exact marginal-likelihood gradient for the conjugate contribution.
The flow state and bridge samples are treated as fixed when taking this derivative, so gradients are not propagated through the preceding flow trajectory.
After the optimiser step, quantities depending on $\boldsymbol\theta$ are recomputed without retaining the computational graph.
In basis space this refreshes the feature matrix and posterior mean, preventing subsequent SDE steps from using a decoder or likelihood built at stale hyperparameters.
After obtaining $\widehat{\boldsymbol\theta}$, we rerun the conditional flow with $\widehat{\boldsymbol\theta}$ fixed to obtain samples from the corresponding fitted conditional model.

\subsection{A Static Interpretation}
\label{subsec:static-interpretation}

We now relate our hyperparameter objective to the standard free-energy interpretation of the log normalising constant.
Specialising this identity to our latent GP representation provides a useful interpretation of the quantities balanced by evidence maximisation.

For simplicity, we here return to the generic evidence objective in Eq.~\eqref{eq:methods-hyperparameter-objective}, regarding \(\mathcal C\) as the complete conditioning statement.
Let
\(
    Z_{\boldsymbol\theta}
    \coloneqq
    Z_{\boldsymbol\theta,R}
    =
    p_{\boldsymbol\theta}(\mathcal C)
\)
and
\(
    J(\boldsymbol\theta)
    =
    \log Z_{\boldsymbol\theta}.
\)
Further, as in Eq.~\eqref{eq: weight-space-gp}, let
\(
    r(\boldsymbol\xi)
    =
    \mathcal N(\boldsymbol\xi;\mathbf0,\mathbf I_R)
\)
denote the standard latent Gaussian density, and let \(p_{\boldsymbol\theta}\) denote the law induced by \(\widetilde f_{\boldsymbol\theta}=g_{\boldsymbol\theta}(\boldsymbol\xi_0)\) when \(\boldsymbol\xi_0\sim r\).
The corresponding latent target is
\begin{equation}
    \rho_{\boldsymbol\theta}(\boldsymbol\xi_0\mid\mathcal C)
    =
    \frac{
        p\!\left(
            \mathcal C
            \mid
            g_{\boldsymbol\theta}(\boldsymbol\xi_0)
        \right)
        r(\boldsymbol\xi_0)
    }{
        Z_{\boldsymbol\theta}
    }.
    \label{eq:app-latent-target}
\end{equation}
By the standard Gibbs variational identity, evaluated at the conditioned law, the evidence balances satisfaction of the condition against the relative-entropy cost of deforming the Gaussian latent prior.
In our setting, whenever the relevant quantities are finite,
\begin{align}
    \log Z_{\boldsymbol\theta}
    =
    \mathbb E_{\boldsymbol\xi_0\sim\rho_{\boldsymbol\theta}(\cdot\mid\mathcal C)}
    \big[
        \log
        p \big(
            \mathcal C
            \mid
            g_{\boldsymbol\theta}(\boldsymbol\xi_0)
        \big)
    \big]
    -
    \operatorname{KL}\bigl(
        \rho_{\boldsymbol\theta}(\cdot\mid\mathcal C)
        \|r
    \bigr).
    \label{eq:static-free-energy-latent}
\end{align}
This follows from the identity
\(
\smash{
    \tfrac{\rho_{\boldsymbol\theta}(\boldsymbol\xi_0\mid\mathcal C)}{r(\boldsymbol\xi_0)}
    =
    \tfrac{p(\mathcal C\mid g_{\boldsymbol\theta}(\boldsymbol\xi_0))}
    {Z_{\boldsymbol\theta}}
}
\), taking logarithms, and taking expectations with respect to $\rho_{\boldsymbol{\theta}}(\cdot|\mathcal{C})$.
There is an equivalent interpretation directly in the GP output space.
Let \(\pi_{\boldsymbol\theta}(\cdot\mid\mathcal C)\) denote the corresponding conditioned law of \(\widetilde f_{\boldsymbol\theta}=g_{\boldsymbol\theta}(\boldsymbol\xi_0)\).
We then have
\begin{align}
    \log Z_{\boldsymbol\theta}
    =
    \mathbb E_{\widetilde f\sim\pi_{\boldsymbol\theta}(\cdot\mid\mathcal C)}
    \big[
        \log p(\mathcal C\mid\widetilde f)
    \big]
    -
    \operatorname{KL}\bigl(
        \pi_{\boldsymbol\theta}(\cdot\mid\mathcal C)
        \|p_{\boldsymbol\theta}
    \bigr).
    \label{eq:static-free-energy-process}
\end{align}
Thus, the log-evidence represents a balance between satisfaction of \(\mathcal C\) and the relative-entropy cost of deforming the GP prior \(p_{\boldsymbol\theta}\) into the conditioned GP law \(\pi_{\boldsymbol\theta}(\cdot\mid\mathcal C)\).

\paragraph{Hard Constraints.}
Suppose that the additional condition \(\mathcal C\) is the event that the approximate GP sample lies in a feasible set \(A\), so that
\(
    p(\mathcal C\mid\widetilde f)
    =
    \mathbf 1_A(\widetilde f).
\)
If \(p_{\boldsymbol\theta}(A)>0\), then
\(
    Z_{\boldsymbol\theta}
    =
    p_{\boldsymbol\theta}(A)
\)
and
\(
    \pi_{\boldsymbol\theta}(\cdot\mid\mathcal C)
    =
    p_{\boldsymbol\theta}(\cdot\mid A).
\)
Since
\(
    \mathrm d p_{\boldsymbol\theta}(\cdot\mid A)
    /
    \mathrm d p_{\boldsymbol\theta}
    =
    \mathbf 1_A/p_{\boldsymbol\theta}(A),
\)
we obtain
\begin{equation}
    \log Z_{\boldsymbol\theta}
    =
    -\operatorname{KL}\bigl(
        p_{\boldsymbol\theta}(\cdot\mid A)
        \|p_{\boldsymbol\theta}
    \bigr).
    \label{eq:hard-event-evidence-kl}
\end{equation}
Thus, for a positive-probability hard constraint, evidence maximisation admits the particularly simple interpretation of minimising the KL deformation from the unconditioned GP law to the constrained GP law.

\paragraph{Soft Constraints.}
Similarly, suppose that the additional condition is encoded by a likelihood of the form
\(
    p(\mathcal C\mid\widetilde f)
    =
    a\exp\{-\Phi_{\mathcal C}(\widetilde f)\},
\)
where \(\Phi_{\mathcal C}(\widetilde f)\geq0\) and \(a>0\) is independent of \(\boldsymbol\theta\).
Substituting
\(
    \log p(\mathcal C\mid\widetilde f)
    =
    \log a-\Phi_{\mathcal C}(\widetilde f)
\)
into Eq.~\eqref{eq:static-free-energy-process} gives, up to additive constants independent of \(\boldsymbol\theta\),
\begin{equation}
    -\log Z_{\boldsymbol\theta}
    \stackrel{+c}{=}
    \operatorname{KL}\bigl(
        \pi_{\boldsymbol\theta}(\cdot\mid\mathcal C)
        \|p_{\boldsymbol\theta}
    \bigr)
    +
    \mathbb E_{\widetilde f\sim\pi_{\boldsymbol\theta}(\cdot\mid\mathcal C)}
    \big[
        \Phi_{\mathcal C}(\widetilde f)
    \big].
    \label{eq:soft-residual-evidence-kl}
\end{equation}
Thus, for soft conditions, evidence maximisation admits the familiar free-energy interpretation of balancing condition satisfaction against the KL deformation of the unconditioned GP law.

\subsection{A Dynamic Interpretation}
\label{subsec:dynamic-interpretation}

The static free-energy identity above expresses evidence maximisation in terms of condition satisfaction and a KL deformation cost.
A standard stochastic-control interpretation of relative entropy, based on Doob transforms and Girsanov's theorem, relates this deformation cost to the quadratic energy of the corresponding controlled drift \citep{doob1957conditional,girsanov1960transforming,oksendal2003stochastic,boue1998variational}.
We now specialise this construction to the latent Ornstein--Uhlenbeck flow underlying \textsc{FeatureFlow}.

Recall from Eq.~\eqref{eq:app-latent-target} that the latent target is
$
\smash{
    \rho_{\boldsymbol\theta}(\boldsymbol\xi\mid\mathcal C)
    =
    \tfrac{
        p\!\left(
            \mathcal C
            \mid
            g_{\boldsymbol\theta}(\boldsymbol\xi)
        \right)
        r(\boldsymbol\xi)
    }{
        Z_{\boldsymbol\theta}
    }
}
$.
Rather than tilting the static latent law \(r\) by \(p(\mathcal C\mid g_{\boldsymbol\theta}(\boldsymbol\xi))\), one can consider a stationary reference diffusion in latent space, and tilt its path law by the likelihood of its terminal endpoint.
We work on a finite generation horizon, rescaled to $[0,1]$.
Let $\mathsf P$ denote the path law of the stationary latent Ornstein--Uhlenbeck process
\begin{equation}
    \mathrm d\boldsymbol U_s
    =
    -\frac12\bar\beta(s)\boldsymbol U_s\,\mathrm ds
    +
    \sqrt{\bar\beta(s)}\,\mathrm d\boldsymbol B_s,
    \qquad
    \boldsymbol U_0\sim r,
    \qquad
    0\leq s\leq1,
    \label{eq:reference-generation-sde}
\end{equation}
where $\bar\beta\in L^1([0,1])$ is nonnegative.
Since the process is stationary, $\boldsymbol U_s\sim r$ for every $s$.
Define the terminally tilted path law $\mathsf Q_{\boldsymbol\theta}$ by
\begin{equation}
    \frac{\mathrm d\mathsf Q_{\boldsymbol\theta}}
    {\mathrm d\mathsf P}
    =
    \frac{
        p\!\left(
            \mathcal C
            \mid
            g_{\boldsymbol\theta}(\boldsymbol U_1)
        \right)
    }{
        Z_{\boldsymbol\theta}
    }.
    \label{eq:path-tilt}
\end{equation}
Its terminal marginal is $\rho_{\boldsymbol\theta}(\cdot\mid\mathcal C)$.
The corresponding Doob potential is
$
    h_{\boldsymbol\theta}(s,\boldsymbol u)
    =
    \mathbb E_{\mathsf P}
    [
        p(
            \mathcal C
            \mid
            g_{\boldsymbol\theta}(\boldsymbol U_1)
        )
        \mid
        \boldsymbol U_s=\boldsymbol u
    ].
$
This is the generation-time counterpart of the smoothed likelihood used for guidance.
Under standard regularity conditions for the Doob transform and Girsanov's theorem, the dynamics under $\mathsf Q_{\boldsymbol\theta}$ are
\begin{equation}
    \mathrm d\boldsymbol U_s
    =
    \left(
        -\frac12\bar\beta(s)\boldsymbol U_s
        +
        \bar\beta(s)
        \nabla_{\boldsymbol u}
        \log h_{\boldsymbol\theta}(s,\boldsymbol U_s)
    \right)\mathrm ds
    +
    \sqrt{\bar\beta(s)}
    \mathrm d\boldsymbol B_s^{\mathsf Q_{\boldsymbol\theta}}.
    \label{eq:guided-sde-under-Q}
\end{equation}
Thus $\nabla_{\boldsymbol u}\log h_{\boldsymbol\theta}$ is the generation-time counterpart of the guidance term used in the conditional flow.
The standard Girsanov relative-entropy identity then gives  \citep[e.g.,][]{girsanov1960transforming,oksendal2003stochastic}
\begin{equation}
    \frac12
    \mathbb E_{\mathsf Q_{\boldsymbol\theta}}
    \left[
        \int_0^1
        \bar\beta(s)
        \left\|
            \nabla_{\boldsymbol u}
            \log h_{\boldsymbol\theta}(s,\boldsymbol U_s)
        \right\|^2
        \mathrm ds
    \right]
    =
    \operatorname{KL}\bigl(
        \rho_{\boldsymbol\theta}(\cdot\mid\mathcal C)
        \|r
    \bigr)
    -
    \operatorname{KL}\bigl(
        (\mathsf Q_{\boldsymbol\theta})_0
        \|\mathsf P_0
    \bigr).
    \label{eq:energy-with-boundary}
\end{equation}

Suppose now that we define the guidance energy by
\begin{equation}
    \mathcal E_{\mathrm{guide}}(\boldsymbol\theta)
    =
    \frac12
    \mathbb E_{\mathsf Q_{\boldsymbol\theta}}
    \left[
        \int_0^1
        \bar\beta(s)
        \left\|
            \nabla_{\boldsymbol u}
            \log h_{\boldsymbol\theta}(s,\boldsymbol U_s)
        \right\|^2
        \mathrm ds
    \right].
    \label{eq:guide-energy-def}
\end{equation}
Gaussian entropy contraction for the Ornstein--Uhlenbeck semigroup gives \citep{gross1975logarithmic,bakry2014analysis}
$
    \operatorname{KL}(
        (\mathsf Q_{\boldsymbol\theta})_0
        \|\mathsf P_0
    )
    \leq
    \exp\{
        -\int_0^1\bar\beta(s)\,\mathrm ds
    \}
    \operatorname{KL}(
        \rho_{\boldsymbol\theta}(\cdot\mid\mathcal C)
        \|r
    ).
$
Hence, whenever $\operatorname{KL}(\rho_{\boldsymbol\theta}(\cdot\mid\mathcal C)\|r)<\infty$, the boundary term vanishes in the complete-noising limit $\int_0^1\bar\beta(s)\,\mathrm ds\to\infty$.
In this case, Eq.~\eqref{eq:energy-with-boundary} then yields
\[
    \mathcal E_{\mathrm{guide}}(\boldsymbol\theta)
    =
    \operatorname{KL}\bigl(
        \rho_{\boldsymbol\theta}(\cdot\mid\mathcal C)
        \|r
    \bigr).
\]
Thus, combining this with Eq.~\eqref{eq:static-free-energy-latent}, we have that
\begin{equation}
    \log Z_{\boldsymbol\theta}
    =
    \mathbb E_{
        \boldsymbol\xi_0
        \sim
        \rho_{\boldsymbol\theta}(\cdot\mid\mathcal C)}
    \left[
        \log
        p\!\left(
            \mathcal C
            \mid
            g_{\boldsymbol\theta}(\boldsymbol\xi_0)
        \right)
    \right]
    -
    \mathcal E_{\mathrm{guide}}(\boldsymbol\theta).
    \label{eq:evidence-guidance-energy-boundary-free}
\end{equation}
It follows, applying the standard control-energy representation to our latent flow, that log-evidence maximisation can be interpreted dynamically as balancing condition satisfaction against the guidance energy required to deform the reference process.
For positive-probability hard constraints, the log-likelihood contribution vanishes on the constrained support, so, in the complete-noising limit, evidence maximisation reduces to minimising the guidance energy.
For soft constraints, the objective also includes the expected residual loss in Eq.~\eqref{eq:soft-residual-evidence-kl}.

\section{Experimental Details \& Additional Results}
\label{app: add-results}

All reported experiments use the CPU on an Apple M5 Pro with 24\,GB memory, using Python 3.12.13, PyTorch 2.13.0 \citep{paszke_pytorch_2019}, and GPyTorch 1.15.2 \citep{gardner2018gpytorch}. 
Runtime values are therefore hardware-dependent.

\subsection{Probabilistic Downscaling}
\label{app: downscaling}

This experiment compares a dense function-space GP flow with \textsc{FeatureFlow}, with increasing spatial resolution. 
Both methods use the same fitted kernel, observations, non-linear constraints, diffusion schedule, and sampling budget; only the sampling representation differs.

\subsubsection{Ground Truth and Observations}

The spatial domain is $[0,1]^2$. Each $H\times H$ evaluation grid contains $N=H^2$ equally spaced points, including the domain boundary. 
We use a zero-mean GP with an RBF kernel \citep{rasmussen_gaussian_2006}:
\begin{equation}
  k(x,x')=\sigma_f^2
  \exp\!\left(-\frac{\lVert x-x'\rVert^2}{2\ell^2}\right).
\end{equation}
The kernel is fitted once to 17 anchor observations and then frozen. 
Nine anchors have value $+5$ at $(a,a)$ for $a\in A = \{0.1,0.2,\ldots,0.9\}$, while eight have value $-5$ at $(a,1-a)\ \forall\ a\in\ A\textbackslash\{0.5\}$.
The observation-noise variance is $10^{-4}$. Starting from $\ell=0.3$ and $\sigma_f^2=1$, the hyperparameters are optimised conventionally by exact marginal likelihood for 500 Adam iterations with learning rate $0.1$, while the zero mean and noise variance remain fixed. 
The fitted values are $\ell=0.096$ and $\sigma_f^2=2.486$.

The ground truth $f^\star$ is one fixed draw from a scalable approximation to the anchor-conditioned posterior.
The posterior is represented using 4,096 RFFs and the 17 additional observation-noise features required by the exact Matheron's rule correction \citep{rahimi_random_2007, wilson_efficiently_2020}.
We then sample 32 locations $\{x_m\}_{m=1}^{32}$ uniformly from $[0,1]^2$ and observe $y_m=f^\star(x_m)$ without adding further noise. 
Both methods target the same GP posterior conditioned on these 32 pairs, using a Gaussian observation model with variance $10^{-4}$.

The GP flow evaluates the exact posterior mean and covariance on each grid and samples an $N$-dimensional whitened representation. 
Its covariance square root is computed by eigendecomposition with jitter $10^{-12}$. 
\textsc{FeatureFlow} instead uses 1,024 random Fourier features and 32 observation-noise features, giving a fixed latent width of 1,056.

\subsubsection{Non-linear Conditions}

Each grid is divided into a regular $8\times8$ partition. For region $A_q$ of the coarse grid, the conditioning targets are the discrete maximum and minimum of the ground truth in that grid cell:
\begin{equation}
  y_q^{\max}=\max_{x_i\in A_q} f^\star(x_i),
  \qquad
  y_q^{\min}=\min_{x_i\in A_q} f^\star(x_i).
\end{equation}
The targets are recomputed at each resolution. 
The condition likelihood is
\begin{equation}
  \begin{split}
  \log p(\mathcal{C}\mid f)
  =-\frac{1}{2\sigma_{\mathcal{C}}^2}\sum_{q=1}^{64}
  \biggl[
    &\left(\max_{x_i\in A_q}f(x_i)-y_q^{\max}\right)^2\\
    +&\left(\min_{x_i\in A_q}f(x_i)-y_q^{\min}\right)^2
  \biggr],
  \end{split}
\end{equation}
with $\sigma_{\mathcal{C}}=10^{-3}$. 
Thus each sample is guided by 64 maximum and 64 minimum residuals. 
The guidance for a GP flow differentiates this condition with respect to the whitened grid values; \textsc{FeatureFlow} differentiates the same condition through the feature map with respect to its weights.

\subsubsection{Sampling and resolution sweep}

Both methods use the same linear variance-preserving diffusion schedule.
For each method and resolution, we generate 50 samples using 1,000 stochastic Euler--Maruyama steps on a log-SNR time grid from $t=1$ to $t=10^{-10}$. 
Guidance uses one Monte Carlo terminal draw, with its norm smoothly clipped to 200. 
All experiments use a fixed seed for reproducibility.

For display in Fig. \ref{fig: downscale-results}, the function-space side lengths are $H\in\{16,24,32,40,48,56,64,72,88,96\}$. 
\textsc{FeatureFlow} additionally uses $H\in\{112,128,144,168,192,224,256\}$.

\subsubsection{Timing and Evaluation}

Setup time for the GP flow includes evaluating the posterior mean and dense $N\times N$ covariance and computing its eigendecomposition.
\textsc{FeatureFlow} setup involves constructing the fixed-width weight-space flow, evaluating the basis on the grid, and constructing the two condition objects.
Kernel fitting, generation of the ground truth and point observations, and construction of the regional targets are excluded from timing.

For sample $b$, field accuracy is measured by
\begin{equation}
  \operatorname{RMSE}_b
  =\left[
    \frac{1}{N}\sum_{i=1}^{N}
    \bigl(f_b(x_i)-f^\star(x_i)\bigr)^2
  \right]^{1/2}.
\end{equation}
In Fig. \ref{fig: downscaling-sample-quality}, we report the mean RMSE over the 50 samples, measured across a range of different numbers of features $R$, with all other parameters held fixed ($H=W=96$, $N=9216$).

\begin{figure}
    \centering
    \includegraphics[width=0.8\linewidth]{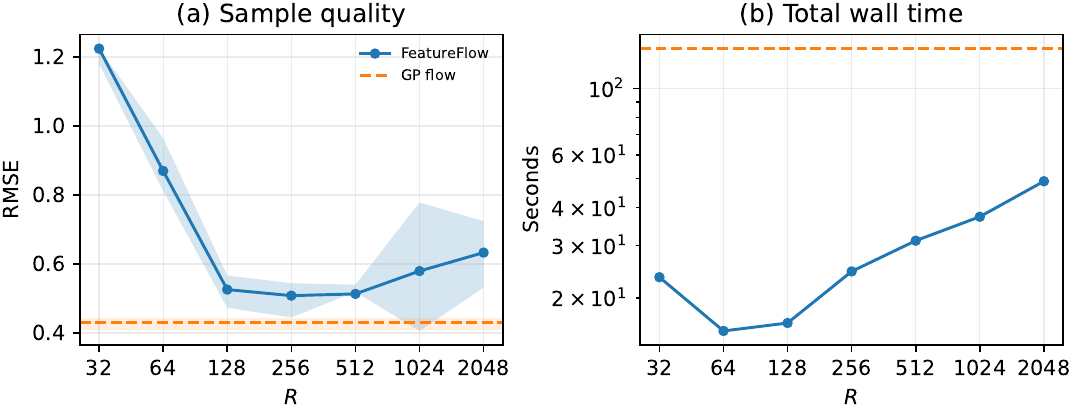}
    \caption{\textbf{(a)} Comparison of sample quality between the GP flow and \textsc{FeatureFlow}, as a function of the number of features $R$. \textbf{(b)} Total wall time of sampling using each method; on this problem \textsc{FeatureFlow} trades off a small decrease in sample quality for large computational speed-up.}
    \label{fig: downscaling-sample-quality}
\end{figure}

In Fig. \ref{fig: downscale-results}, we show the measured setup times together with reference slopes $O(N^3)$ for the dense GP flow and $O(N)$ for \textsc{FeatureFlow}. 
These are fixed reference exponents rather than fitted slopes.
For the 0.5-second setup-budget comparison, the highest measured resolution within budget is $40\times40$ for function space and $144\times144$ for \textsc{FeatureFlow}.

\subsection{PDE Solutions}
\label{app:pde}

The spatial domain $\Omega$ is the interior of a discrete star-shaped mask, embedded in the unit square $[0,1]^2$ (shape $H \times W = H^2$). 
The mask was retrieved from the code accompanying \cite{solin_know_2019}, which is licensed under Apache 2.0. 

We consider the stationary Allen--Cahn equation
\begin{align}
  a f(x)-bf(x)^3-c\,\Delta f(x)&=0 &&x\in\Omega,
  \label{eq:allen-cahn-app}
  \\
  f(x)&=0 &&x\in \partial \Omega,
\end{align}
with $(a,b,c)=(-1,-0.2,2\times 10^{-3})$.
Equivalently, the physical residual is $r_{\mathrm{PDE}}(f)=-f+0.2\times f^3- (2\times 10^{-3})\Delta f$. 

\paragraph{Reference field.}
The observation-generating field was generated via a numerical solve using a nine-point Laplacian on a $64\times64$ grid. 
With $h=1/63$, the nine-point Laplacian is
\begin{equation}
\begin{split}
 (\Delta_h f)_{ij}
 &=\frac{1}{6h^2}\bigg[
 4\!\sum_{(p,q)\in\mathcal A}f_{i+p,j+q}\\
 &\quad+\sum_{(p,q)\in\mathcal B}f_{i+p,j+q}
 -20f_{ij}\bigg],
\end{split}
 \label{eq:nine-point-laplacian}
\end{equation}
where $\mathcal A=\{(-1,0),(1,0),(0,-1),(0,1)\}$ and
$\mathcal B=\{(-1,-1),(-1,1),(1,-1),(1,1)\}$.

The eighth smallest Dirichlet eigenmode of $-c\Delta_h$ was used to initialise the non-linear solve, in order to converge to a high-frequency solution.
A damped Newton method then solves Eq.~\eqref{eq:allen-cahn-app}. 
The RMS-residual tolerance is $10^{-10}$, the maximum number of Newton updates is 80, and the smallest permitted line-search step is $2^{-24}$. 
The resulting $64\times64$ array is used for observations.

\subsubsection{Observations and GP posterior}
\label{app:pde-data-prior}

We draw $M=16$ distinct observation locations $X_y = \{\mathbf x_m\}_{m=1}^M$ uniformly without replacement from the array of the numerical solve, which we treat as the ground truth. 
Measurements obey
\begin{equation}
 y_m=f_\star(\mathbf x_m)+\epsilon_m,
 \qquad \epsilon_m\overset{\mathrm{iid}}{\sim}\mathcal N(0,0.01^2).
\end{equation}

The zero-mean GP prior has an isotropic scaled radial-basis-function kernel
\begin{equation}
 k_{\boldsymbol \theta}(x,x')=\sigma_f^2 \exp\!\left(-\frac{\lVert x-x'\rVert^2}{2\ell^2}\right),
 \label{eq:pde-rbf}
\end{equation}
The observation noise variance is fixed at $10^{-4}$. 
At every value of $\ell$, the prior is conditioned analytically on the $M$ observations:
\begin{align}
 C_{\boldsymbol \theta}&=K_{\boldsymbol \theta}(X_y,X_y)+10^{-4}\mathbf I,\\
 m_{\boldsymbol \theta}(x)&=k_{\boldsymbol \theta}(x,X_y)C_{\boldsymbol \theta}^{-1}y,\\
 k_{\boldsymbol \theta}^{\mathrm{post}}(x,x')&=k_{\boldsymbol \theta}(x,x')-
 k_{\boldsymbol \theta}(x,X_y)C_{\boldsymbol \theta}^{-1}k_{\boldsymbol \theta}(X_y,x').
\end{align}
Every run starts from $\ell=0.3$. 
The output scale is frozen at $\sigma_f^2=1.0$ during adaptation, so the length scale is the only optimized kernel parameter. 

\subsubsection{Function-space and Feature-space Representations}
\label{app:pde-representations}

\paragraph{GP flow.}
At a grid of $N=H^2$ points, the dense method evaluates the posterior mean and covariance, computes a symmetric eigendecomposition $K_{\boldsymbol \theta}^{\mathrm{post}}=U_{\boldsymbol \theta}\Lambda_{\boldsymbol \theta} U_{\boldsymbol \theta}^\top$, and uses $L=U_{\boldsymbol \theta}\Lambda_{\boldsymbol \theta}^{1/2}$ to map a whitened state $z\in\mathbb R^N$ to $f_{\boldsymbol \theta}=m+L_{\boldsymbol \theta}z$. 
A diagonal jitter of $10^{-12}$ is added before the eigendecomposition and eigenvalues are lower-clamped at $10^{-12}$. 
During adaptation this dense decoder is recomputed after each length scale update.

\paragraph{\textsc{FeatureFlow}.}
The feature-space method uses the kernel approximation defined by
\begin{equation}
 f_{\boldsymbol \theta}(x)=m(x)+\boldsymbol\Psi^{\mid y}_{\boldsymbol \theta}(x)^\top \hat{\boldsymbol\xi},
 \qquad \hat{\boldsymbol\xi}\sim\mathcal N(0,\mathbf I_{R+M}).
\end{equation}
The first $R = 256$ features are Dirichlet Laplacian eigenfunctions on the star
domain. 
The remaining $|X_y| = M = 16$ features implement the observation-noise block in a Matheron correction \citep{wilson_efficiently_2020}; thus, the total diffusion dimension is $R=272$.

The eigenbasis is constructed numerically once on a fine $256\times256$ star mask, according to \citet{solin_know_2019}. 
The exact RBF kernel supplies the rank-16 pathwise conditioning correction and posterior mean. 
Both the spectral amplitudes and conditioning correction are refreshed after every adaptive update. 

\subsubsection{PDE and Boundary Likelihoods}
\label{app:pde-likelihoods}

There are two conditions, $\mathcal{C}\in\{\mathrm{PDE},\partial\Omega\}$: the PDE
condition and the boundary condition. 
Each has log likelihood
\begin{equation}
 \ell_\mathcal{C} = \log p(\mathcal{C}\mid f_{\boldsymbol \theta})=-\frac{1}{2A_\mathcal{C}}\sum_{j=1}^{A_\mathcal{C}}
 \left(\frac{r_{\mathcal{C},j}(f_{\boldsymbol \theta})}{\sigma_\mathcal{C}}\right)^2.
 \label{eq:component-average}
\end{equation}
Here, $A_\mathcal{C}$ is the number of residual components, and $\sigma_C$ is the dimensionless constant that defines the ``sharpness'' of the condition. 
The total constraint log likelihood is $\ell_{\mathrm{cond}}=\ell_{\mathrm{PDE}}+\ell_{\partial\Omega}$.

For the PDE residual, $\sigma_{\mathrm{PDE}}=10^{-3}$, and for the boundary residual, $\sigma_{\partial\Omega}=10^{-1}$. 
Setting $\sigma_{\partial\Omega}$ to larger values for the GP flow leads to very poor PDE residuals.
The boundary is the one-pixel exterior ring whose pixels have at least one four-connected neighbour in $\Omega$, and its target value is zero, in keeping with the boundary condition in Eq. \eqref{eq:allen-cahn-app}.

The dense GP flow evaluates the PDE residual at interior stencil centres using
Eq.~\eqref{eq:nine-point-laplacian}, with $h=1/(H-1)$.  The sampled field is
defined on the whole square; consequently, stencil neighbours outside $\Omega$
take their sampled values, while the separate exterior-ring likelihood steers
the adjacent values toward zero. 
\textsc{FeatureFlow} instead evaluates
\begin{equation}
 r_{\mathrm{PDE}}(f_{\boldsymbol \theta};x)
 =a f_{\boldsymbol \theta}(x)-b f_{\boldsymbol \theta}(x)^3-c\,\Delta f_{\boldsymbol \theta}(x)
\end{equation}
at every interior grid point, using the eigenvalues and analytic Laplacian of
the pathwise RBF correction. 
In Fig. \ref{fig: pde-results}, the two plotted PDE MSEs therefore use the same physical equation but method-specific operators (nine-point finite difference for the GP flow and analytic for \textsc{FeatureFlow}).

\subsubsection{Guided Sampling}
\label{app:pde-sampling}

Every run uses \texttt{num\_steps=1000}.
The sampler constructs 1,001 scheduled times at equal increments of log signal-to-noise ratio between $t=1$ and $t=10^{-10}$. 
It uses a linear variance-preserving schedule and stochastic Euler--Maruyama integration. 

Likelihood guidance uses $S=5$ terminal Monte Carlo draws per particle and step.
The Monte Carlo noise is drawn once and held fixed along each trajectory. 
For stability, the guidance vector is smoothly norm-limited using a $\tanh$ function. 
The maximum norm is 200 for the function-space GP flow and 100 for the weight-space \textsc{FeatureFlow}.

Every fixed-prior run generates 50 particles in one batch. 
The adaptive runs use only 5 particles.
In summary, the common sampling configuration is: float64 CPU calculation; \texttt{num\_steps=1000}; equal log-SNR spacing on $t\in[10^{-10},1]$; a linear variance-preserving schedule; stochastic Euler--Maruyama integration with final denoising; 50 particles per fixed-prior ensemble; five particles during adaptation; five guidance terminal draws; and guidance norm limits of 200 for the GP flow and 100 for \textsc{FeatureFlow}.

\subsubsection{Online Hyperparameter Adaptation}
\label{app:pde-adaptation}

Adaptive runs start from $\ell=0.3$ and update the lengthscale every 10 integration steps using Adam with learning rate $0.05$. 
Hence, the 999 integrated steps produce 99 parameter updates. 
A separate set of $S=5$ frozen terminal draws is used in the adaptation objective. 
At an update time, one update step is taken according to App. \ref{app: hyperparam-opt}, in order to maximise the objective.
Hyperparameter trajectories during the adaptive sampling run are shown in Fig. \ref{fig: pde-param-histories}, for both the GP flow and \textsc{FeatureFlow}.
We note that the optimisation dynamics in \textsc{FeatureFlow} appear more stable than those in the vanilla GP flow.

\begin{figure}
    \centering
    \includegraphics[width=\linewidth]{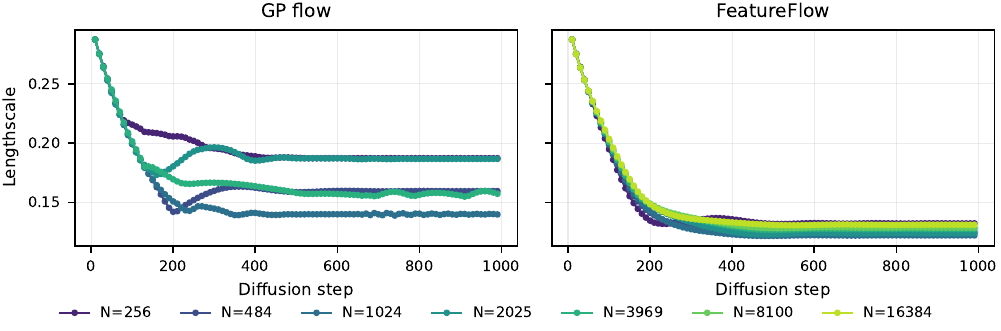}
    \caption{Hyperparameter adaptation histories in both function (\textbf{left}) and weight space (\textbf{right}) for the PDE problem.}
    \label{fig: pde-param-histories}
\end{figure}

Table~\ref{tab:pde-final-lengthscales} gives the realised final values. 
The output scale is 1.0 in each case.

\begin{table}[t]
\centering
\caption{Final fitted lengthscales. A dash denotes a GP flow resolution that was not run because of its dense covariance cost.}
\label{tab:pde-final-lengthscales}
\small
\begin{tabular}{r r c c}
\hline
$H=W$ & $N$ & GP flow & \textsc{FeatureFlow} \\
\hline
 16  &   256 & 0.1872 & 0.1322 \\
 22  &   484 & 0.1595 & 0.1248 \\
 32  &  1024 & 0.13978 & 0.1222 \\
 45  &  2025 & 0.1866 & 0.1240 \\
 63  &  3969 & 0.1572 & 0.1275 \\
 90  &  8100 & ---          & 0.1288 \\
128  & 16384 & ---          & 0.1311 \\
\hline
\end{tabular}
\end{table}

\subsubsection{Resolution Ablation, Metrics, and Timing}
\label{app:pde-evaluation}

We consider side lengths $H\in\{16,22,32,45,63,90,128\}$. 
\textsc{FeatureFlow} is run at all seven resolutions; the dense GP flow is restricted to $H\leq64$. 
The same 16 observations are used at every resolution.

PDE error is the mean of the squared, unscaled physical residual over all 50 samples and all PDE evaluation points.
Boundary error is the corresponding mean of $f^2$ on the exterior ring. 

Wall time is measured with \texttt{perf\_counter}.  For an initial run, total
time is fixed-run setup plus fixed-run sampling. 
For an adapted result it is the setup and sampling time of the fresh final fixed run plus the complete adaptive run, including adaptive setup and all 99 updates.  
The dotted reference curves do not estimate their exponent: the exponents are fixed to three for GP flow and one for \textsc{FeatureFlow}, with only a multiplicative constant fitted by matching the geometric mean of measured and reference times in log space.

\subsection{Sea Level Anomaly Estimation}
\label{app:sla-estimation}

\subsubsection{Data and Synthetic Observations}
\label{app:altimeter-data}

\paragraph{Truth field and sampling geometry.}
The target field has shape $105\times80$ and is taken from the IMOS OceanCurrent gridded sea-level anomaly (GSLA) analysis for 29 August 2026 in the southeast-Australia domain, provided under a Creative Commons Attribution 4.0 International License \citep{IMOS}. 
In Fig. \ref{fig:raw-gt}, we show a visualisation of the GSLA analysis---which we treat as a high-quality reference field because it assimilates additional observational sources unavailable to our inversion.
We use the tracks of 16 real, distinct satellite passes \citep{IMOS}, with simulated shots taken at $0.288$ km intervals.
The GSLA field is sampled at these locations by bilinear interpolation of the gridded product.
This produces $63{,}458$ shots spanning approximately $18{,}276$ km of satellite track length.

\paragraph{Nuisance fields and waveform generation.}
The sea-surface height at each shot is taken from the interpolated GSLA field. 
The significant wave height (SWH), 10-m wind speed $U_{10}$, and skewness $\lambda$ are synthetic nuisance fields treated as known in the experiment.
The reflectivity-tilt parameter is constant, $\beta=0.12$.

The forward model produces one 104-gate mean waveform per shot.
Conditional on that mean, independent Gamma speckle is generated with 90 looks, resulting in $63{,}458\times104$ observations.

\paragraph{Information supplied to the inversion.}
The inversion observes the raw 104-gate waveform from every shot. 
It treats the four nuisance quantities (SWH, $U_{10}$, $\lambda$, and $\beta$) as known. 
Thus the only unknown field in this experiment was sea-surface height.  
The exact joint Gamma log-likelihood is evaluated over all $6{,}599{,}632$ waveform gates. 
Shots are divided into chunks of 256 solely to limit memory use; chunk log-likelihoods are summed before guidance is calculated, so this introduces no likelihood factorisation or approximation.

\begin{wrapfigure}{r}{0.35\textwidth}
    \centering
    \includegraphics[width=\linewidth]{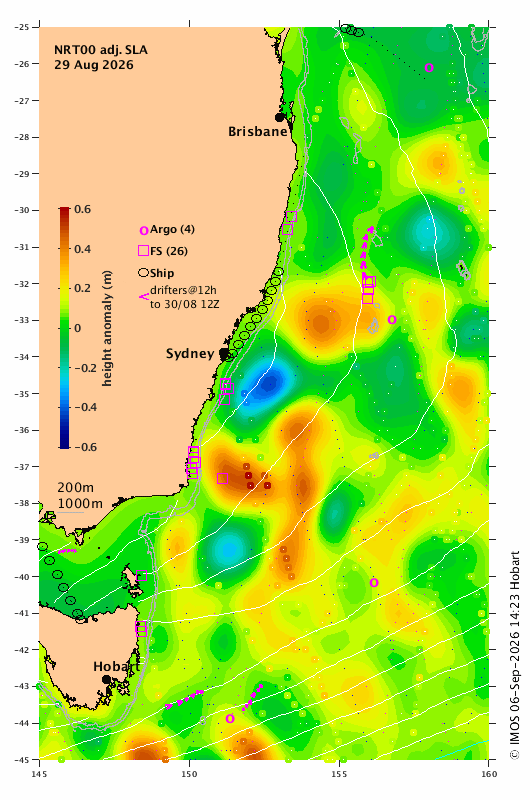}
    \caption{Visualisation of the GSLA analysis provided by \citet{IMOS}. This ground truth SLA field is obtained using a range of observational data as shown, beyond satellite altimetry.}
    \label{fig:raw-gt}
\end{wrapfigure}

\subsubsection{Spatial Prior and Basis Representation}
\label{app:altimeter-prior}

Longitude $\zeta$ and latitude $\varphi$ are mapped to a local planar coordinate system before applying the RBF kernel. 
With $(\zeta_0,\varphi_0)=(152.5^{\circ},-35^{\circ})$,
\begin{equation}
  x=(\zeta-\zeta_0)\times 91.19,
  \quad
  y=(\varphi-\varphi_0)\times 111.32
  \quad
  \mathbf{x} = (x, y)^\top.
\end{equation}

The prior is

\begin{equation}
 f(\mathbf{x})\sim\mathcal{GP}\!\left(
   \mu,\; \sigma_f^2
   \exp\!\left[-\frac{\|\mathbf{x}-\mathbf{x}'\|^2}{2\ell^2}\right]
 \right).
 \label{eq:altimeter-prior}
\end{equation}

For the \textit{oracle} prior, $\mu$, $\sigma_f^2$, and $\ell$ are fitted to the complete truth field before conditioning on waveforms. 
The adaptive experiment retains the oracle mean and variance and treated only the length scale as unknown.

A total of $R=512$ RFFs (256 sampled frequencies) parameterises the field as in \ref{app: rffs}.

The same weight vector is evaluated both on the $105\times80$ display grid and at all shot coordinates. 
The RFF run therefore has 512 latent variables regardless of the number of observations. 
The full rectangular field is sampled, but land cells are masked when computing field metrics.
Future work could investigate appropriate physics-informed boundary conditions, especially near coastlines, that would inform the samples drawn from \textsc{FeatureFlow} beyond the non-Gaussian measurements considered here.

\subsubsection{Guided Sampling}
\label{app:altimeter-sampling}

Every run uses \texttt{num\_steps=100}. 
The sampler constructs 101 scheduled times at equal increments of log signal-to-noise ratio between $t=1$ and $t=10^{-10}$.

Likelihood guidance uses $S=1$ terminal Monte Carlo draw per particle and step. 
For stability, the guidance vector is smoothly norm-limited to a maximum norm of $1000$ using a $\tanh$ function.

The fixed-prior runs each generate 50 particles in one batch. 
In summary, the common sampling configuration is: RFF basis with $R=512$; float64 CPU calculation; \texttt{num\_steps=100}; equal log-SNR spacing on $t\in[10^{-10},1]$; stochastic Euler--Maruyama integration; 50 particles per fixed-prior ensemble; one guidance terminal draw; guidance norm limit 1000; and likelihood chunks of 256 shots.

\subsubsection{Adaptive Length Scale Experiment}
\label{app:altimeter-adaptation}

The adaptive comparison consists of the following three stages (samples from each shown in Fig. \ref{fig:altimeter-all-fields}).

\begin{enumerate}
  \item \textbf{Misspecified fixed prior.}  The initial length scale is set to three times the oracle value, $\ell_{\mathrm{mis}}=3\ell_{\mathrm{oracle}}=226.25$ km. 
  With this value fixed, 50 posterior draws are produced using the sampling configuration above.

  \item \textbf{Adaptive path.}  A single guided trajectory is initialised at the same misspecified length scale. 
  The kernel output variance is frozen at $\sigma_f^2=0.014\ \text{m}^2$, the mean is fixed at $0.097$ m, and only the isotropic RBF length scale is trainable. 
  After every five integration steps, one Adam update \citep{kingma_adam_2017} is applied to the length scale.

  The learning rate is rescaled from $0.07$ to $0.07 \times 226 = 15.82$ using the initial length scale to encourage convergence in a reasonable number of steps. 
  Both guidance and kernel fitting use only one terminal Monte Carlo draw; the kernel-fitting draw noise remains fixed between updates.
  The final fitted length scale obtained is $81.55$ km.

  \item \textbf{Fitted fixed prior.}  The fitted value $81.55$ km is frozen and a
  fresh 50-draw guided ensemble was generated. 
  This is the adaptive method's final posterior ensemble, used to report results. 
\end{enumerate}

The separate oracle fixed run uses a fixed random seed for both its basis construction and sampler. 
Its length scale is $75.42$ km and it also contains 50 draws.

All experiments use a fixed seed for reproducibility.

\begin{figure}
    \centering
    \includegraphics[width=0.9\linewidth]{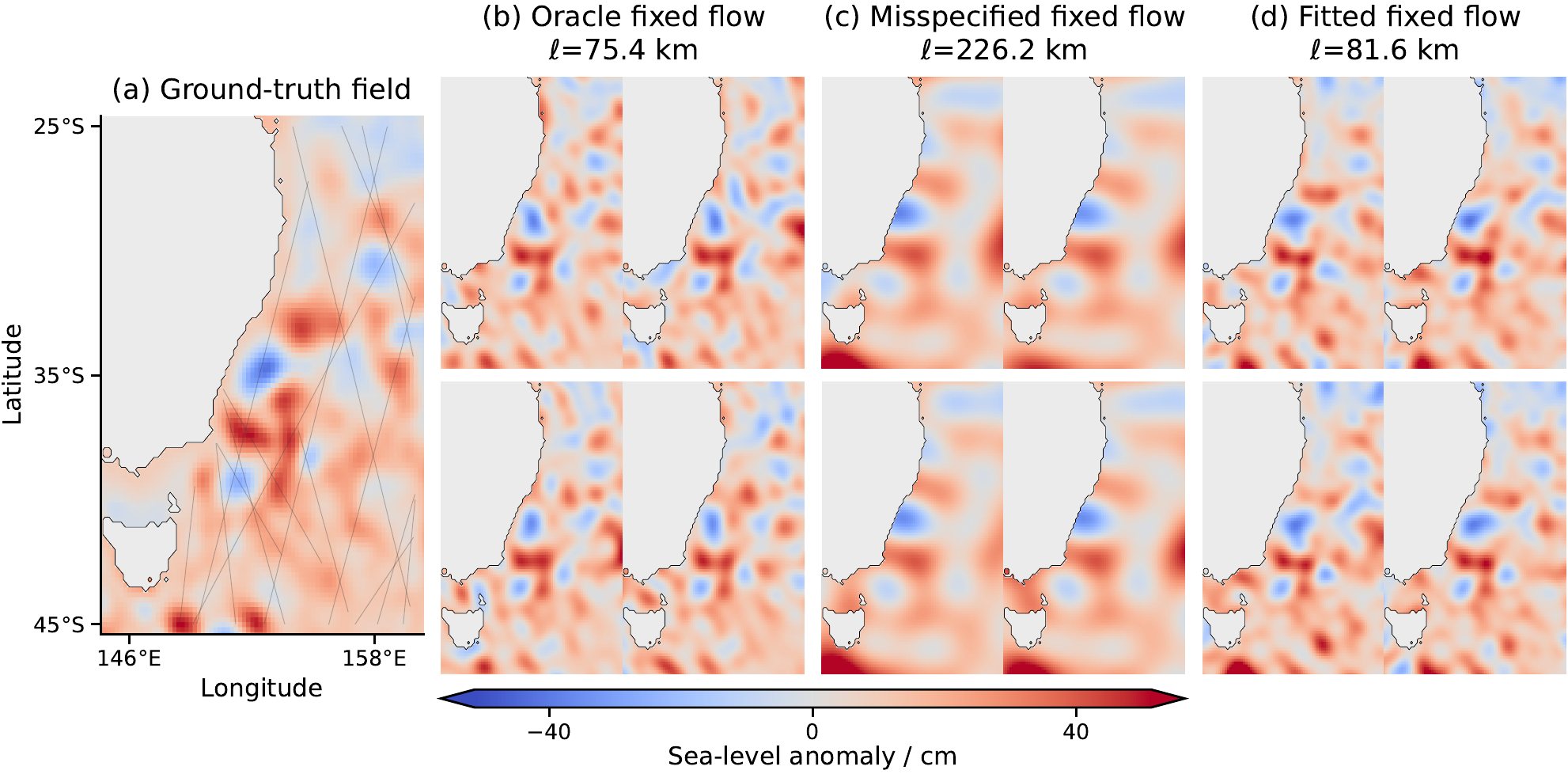}
    \caption{\textbf{(a)} The ground-truth field, with satellite tracks overlaid. Samples using \textsc{FeatureFlow} with \textbf{(b)} oracle, \textbf{(c)} misspecified and \textbf{(d)} fitted kernel length scales $\ell$. The oracle fixed flow represents the best possible performance of \textsc{FeatureFlow} with respect to the adaptive scheme.}
    \label{fig:altimeter-all-fields}
\end{figure}

\subsubsection{Baselines and evaluation}
\label{app:altimeter-evaluation}

The classical baseline is MLE3 retracking applied independently to every noisy waveform \citep{brown_average_1977, rodriguez_altimetry_1988,tourain_benefits_2021}. 
Its three fitted quantities are epoch, SWH, and backscatter; sea-surface height is then obtained from the fitted epoch. 
The retracker converged for all $63{,}458$ waveforms. 
We report the uncorrected MLE3 baseline as the primary comparison. 
A post-hoc debiased diagnostic subtracts the MLE3 bias measured on these same evaluation shots ($-7.408$ cm); because this correction uses the truth, it is not a deployable estimator and is only to be used for comparison.

For an ensemble $\{f^{(m)}(\mathbf{x})\}_{m=1}^{M}$, the point estimate is its sample mean. 
Field RMSE is

\begin{equation}
 \mathrm{RMSE}_{\mathrm{field}}
 =100\times \sqrt{\frac{1}{|\mathcal{O}|}\sum_{i\in\mathcal{O}}
 (\overline{f}(\mathbf{x}_i)-f^\star(\mathbf{x}_i))^2},
\end{equation}

\noindent where $\mathcal{O}$ is the interior of the finite-ocean mask, $\overline{f}$ is the mean field, taken over posterior samples, and $f^\star$ is the ground truth field. 
On-track RMSE uses the analogous average over all $63{,}458$ shots after evaluating each draw's shared RFF weights at the exact shot coordinates. 
MLE3 has no two-dimensional field estimate, so only its on-track error is reported. 
The values provided in Table \ref{tab:altimeter-reference-errors} are reflected in Fig. \ref{fig: altimeter-fields}(d), and a comparison with the \textit{post hoc} debiased MLE3 estimate is provided in Fig. \ref{fig: altimeter-debiased-errors}.

\begin{table*}[t]
\centering
\small
\caption{Observed reconstruction errors. The adaptive-path row is the single changing-prior trajectory used for fitting, not the final ensemble.}
\label{tab:altimeter-reference-errors}
\begin{tabular}{lrrr}
\hline
Method & Draws & Field RMSE / cm & On-track RMSE / cm \\
\hline
Misspecified fixed ($\ell=226.25$ km) & 50 & 12.605 & 8.185 \\
Adaptive path ($226.25\rightarrow81.55$ km) & 1 & 14.024 & 12.449 \\
Fitted fixed ($\ell=81.55$ km) & 50 & 8.173 & \textbf{1.651} \\
Oracle fixed ($\ell=75.42$ km) & 50 & \textbf{7.635} & 1.830 \\
MLE3 & --- & --- & 9.202 \\
\textit{Post hoc} debiased MLE3 & --- & --- & 5.458 \\
\hline
\end{tabular}
\end{table*}

\begin{figure}
    \centering
    \includegraphics[width=0.6\linewidth]{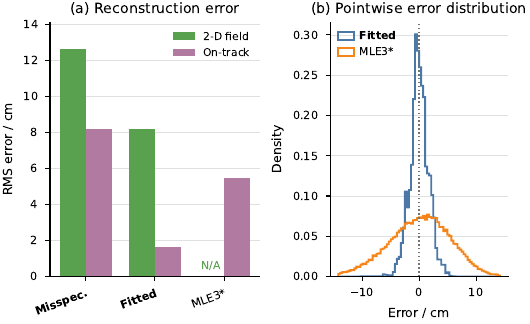}
    \caption{Equivalent results to those shown in Fig. \ref{fig: altimeter-fields}(d, f), only using \textit{post hoc} debiased MLE3 results. \textsc{FeatureFlow} exhibits smaller RMS error even against this more challenging baseline.}
    \label{fig: altimeter-debiased-errors}
\end{figure}




\subsubsection{Forward Model}
\label{app: forward-model}
The forward model used to generate synthetic radar echoes given a ground-truth SLA field is detailed here.
The model has two parts: a deterministic mapping from the SLA field and the nuisance fields to the mean waveform, followed by a non-Gaussian observation model.

For an observation at location $\mathbf{x}_j(\zeta_j,\varphi_j)$, the sea-surface height $f(\mathbf{x}_j)$ determines the two-way delay relative to the tracker reference height $H$:
\begin{equation}
    \tau_j =
    \frac{2}{c}
    \left[
    H-f(\mathbf{x}_j)
    \right],
\end{equation}
where $c$ is the speed of light. The delay associated with range gate $k$ is
\begin{equation}
    t_k =
    \left(k-G_{\mathrm{ref}}\right)\Delta t,
    \qquad
    k=1,\ldots,K,
\end{equation}
where $\Delta t=3.125\,\mathrm{ns}$ (temporal spacing between adjacent range gates) and $G_{\mathrm{ref}}=32.5$ (the range gate index associated with an SLA $f(\mathbf{x}_j)=0$). 
The significant wave height (SWH) determines the temporal width of the scattering surface and its convolution with the radar point-target response:
\begin{equation}
    \sigma_{s,j} =
    \frac{\mathrm{SWH}_j}{2c},
    \qquad
    \sigma_{c,j}^{2} =
    \sigma_{p}^{2}+\sigma_{s,j}^{2},
\end{equation}
where $\sigma_p$ is the width of the point-target response and $\sigma_{c,j}^2$ is the composite waveform width obtained by convolving the scattering-surface distribution with the point-target response. 
Wind speed near the surface ($U_{10}$) determines the backscatter coefficient $\sigma_j^0$ using empirical coefficients $\Sigma_A$ and $\Sigma_B$, and waveform amplitude $A_j$:
\begin{equation}
    \sigma^{0}_j
    =
    \Sigma_A-\Sigma_B
    \log_{10}\!\left(\max\{U_{10,j},1\}\right),
    \quad
    A_j
    =
    P_{\mathrm{ref}}\,
    10^{(\sigma^{0}_j-\sigma^{0}_{\mathrm{ref}})/10}.
\end{equation}
For a Gaussian sea surface at exact nadir, the Brown mean waveform is
\begin{equation}
    B_j(x)
    =
    A_j
    \exp\!\left(
    \frac{\alpha^{2}\sigma_{c,j}^{2}}{2}
    -\alpha x
    \right)
    \Phi\!\left(
    \frac{x}{\sigma_{c,j}}
    -\alpha\sigma_{c,j}
    \right),
\end{equation}
where $x=t_k-\tau_j$, $\alpha$ is the trailing-edge decay rate, and $\Phi$ is the standard Gaussian cumulative distribution function.
Surface skewness $\lambda_j$ and electromagnetic reflectivity-tilt parameter $\beta_j$ are incorporated through the Hermite coefficients
\begin{equation}
    \left(
    q_{0,j},q_{1,j},q_{2,j},q_{3,j},q_{4,j}
    \right)
    =
    \left(
    1,\,
    \beta_j,\,
    -\frac{\lambda_j\beta_j}{2},\,
    -\frac{\lambda_j}{6},\,
    -\frac{\lambda_j\beta_j}{6}
    \right).
\end{equation}
The expected power in gate $k$ of echo $j$ is then
\begin{equation}
    m_{jk}
    =
    P_{\mathrm{thermal}}
    +
    \left.
    \sum_{n=0}^{4}
    q_{n,j}(-1)^n\sigma_{s,j}^{\,n}
    \frac{\mathrm{d}^{\,n}B_j(x)}{\mathrm{d}x^{n}}
    \right|_{x=t_k-\tau_j}
\end{equation}
where $P_{\textrm{thermal}}$ is the minimum power observed due to the thermal floor.
Finally, we model the observed gate powers to follow an exact speckle-noise model using $90$ looks
\begin{equation}
    Y_{jk}\mid m_{jk}
    \sim
    \operatorname{Gamma}\!\left(
    N_l,\frac{m_{jk}}{N_l}
    \right),
    \quad
    N_l=90,
\end{equation}
independently across echoes and range gates; we use the shape-scale convention for the Gamma distribution here.
This parameterisation gives
\begin{equation}
    \mathbb{E}[Y_{jk}\mid m_{jk}]=m_{jk},
    \quad
    \operatorname{Var}[Y_{jk}\mid m_{jk}]
    =
    \frac{m_{jk}^{2}}{N_l}.
\end{equation}
Thus, the forward model maps the SLA and nuisance fields to all $6{,}599{,}632$ observed powers:
\begin{equation}
    \left\{
    f(\mathbf{x}_j),\,
    \mathrm{SWH}_j,\,
    U_{10,j},\,
    \lambda_j,\,
    \beta_j
    \right\}_{j=1}^{63{,}458}
    \mapsto
    \left\{m_{jk}\right\}
    \mapsto
    \left\{Y_{jk}\right\}
\end{equation}
where the first map is deterministic and a function of all nuisance variables, and the second map is a stochastic draw from the Gamma observation model.

\end{document}